%% file: arxiv.tex
\documentclass{article}

\usepackage{graphicx} % more modern
\usepackage{subfigure}
\usepackage{natbib}
\usepackage{algorithm}
\usepackage{algorithmic}
\usepackage{bm}
\usepackage{epstopdf}
\usepackage{fullpage}

\usepackage{amsmath, amsthm, amssymb}

\usepackage{bm}
\usepackage{bbm}
\usepackage{dsfont}
\usepackage{bibentry}
\usepackage{authblk}

\input{symbol.tex}

\title{Efficient Online Bandit Multiclass Learning with $\tilde{O}(\sqrt{T})$ Regret}

\author[1]{Alina Beygelzimer\thanks{beygel@yahoo-inc.com}}
\author[2]{Francesco Orabona\thanks{francesco@orabona.com}}
\author[3]{Chicheng Zhang\thanks{chichengzhang@ucsd.edu}}

\affil[1]{Yahoo Research, New York, NY}
\affil[2]{Stony Brook University, Stony Brook, NY}
\affil[3]{University of California, San Diego, La Jolla, CA}

\begin{document}

\maketitle

\begin{abstract}
We present an efficient second-order algorithm with $\tilde{O}(\frac{1}{\eta}\sqrt{T})$\footnote{$\tilde{O}(\cdot)$ hides logarithmic factors.} regret for the bandit online multiclass problem. The regret bound holds simultaneously with respect to a family of loss functions parameterized by $\eta$, for a range of $\eta$ restricted by the norm of the competitor. The family of loss functions ranges from hinge loss ($\eta=0$) to squared hinge loss ($\eta=1$).  This provides a solution to the open problem of (\bibentry{AbernethyR09}).  We test our algorithm experimentally, showing that it also performs favorably against earlier algorithms.
\end{abstract}

%\footnote{$\tilde{O}(\cdot)$ hides logarithmic factors.}

\input{intro}
\input{setting}
\input{loss}
\input{result}
\input{proof_main}
\input{fallback}
\input{experiments}
\input{discussion}

\paragraph{Acknowledgments.} We thank Claudio Gentile for suggesting the original plan of attack for this problem, Satyen Kale for thought-provoking conversations (which result in Appendix~\ref{sec:expconcave}), and Haipeng Luo for helpful discussions on Newtron. We also thank the anonymous reviewers for thoughtful comments.

\newpage

%appendix
\bibliography{learning}
\bibliographystyle{plainnat}

%\onecolumn
\appendix
\input{adaptive}
\input{proof_perc_q}
\end{document}

%% file: symbol.tex
\mathchardef\mhyphen="2D
\newcommand{\calO}{\mathcal{O}}

\DeclareMathOperator*{\argmax}{arg\,max}
\DeclareMathOperator*{\tr}{tr}

\newcommand{\field}[1]{\mathbb{#1}}

\newcommand{\R}{\field{R}}

\newcommand{\E}{\field{E}}

\newcommand\inn[2]{ \left\langle {#1} \,,\, {#2} \right\rangle }

\newcommand{\norm}[1]{\left\|{#1}\right\|}

\newtheorem{lemma}{Lemma}
\newtheorem{theorem}{Theorem}

\newtheorem{claim}{Claim}
\newtheorem{assumption}{Assumption}

\newcommand{\reals}{\mathbb{R}}

\DeclareMathOperator{\vect}{vec}
\DeclareMathOperator{\mat}{mat}

\newcommand{\one}{\mathds{1}}
\newcommand{\hide}[1]{}
\DeclareMathOperator{\logistic}{logistic}

%% file: intro.tex
\section{Introduction}
\label{sec:intro}
In the online multiclass classification problem, the learner must repeatedly classify examples into one of $k$ classes.  At each step $t$, the learner observes an example $\bm x_t\in \reals^d$ and predicts its label $\tilde{y}_t\in [k]$.  In the full-information case, the learner observes the true label $y_t\in [k]$ and incurs loss $\one[\tilde{y}_t\not= y_t]$.  In the bandit version of this problem, first considered in \cite{Kakade-Shalev-Shwartz-Tewari-2008}, the learner only observes its incurred loss $\one[\tilde{y}_t\not=y_t]$, i.e., whether or not its prediction was correct.  Bandit multiclass learning is a special case of the general contextual bandit learning~\cite{langford_zhang_08} where exactly one of the losses is 0 and all other losses are 1 in every round.

The goal of the learner is to minimize its regret with respect to the best predictor in some reference class of predictors, that is the difference between the total number of mistakes the learner makes and the total number of mistakes of the best predictor in the class.
\citet{Kakade-Shalev-Shwartz-Tewari-2008} proposed a bandit modification of the Multiclass Perceptron algorithm~\citep{DudaH73}, called the Banditron, that uses a reference class of linear predictors.
Note that even in the full-information setting, it is difficult to provide a true regret bound. Instead, performance bounds are typically expressed in terms of the total multiclass hinge loss of the best linear predictor, a tight upper bound on 0-1 loss.

The Banditron, while computationally efficient, achieves only $O(T^{2/3})$ expected regret with respect to this loss, where $T$ is the number of rounds.  This is suboptimal as the Exp4 algorithm of \citet{Auer:2003} can achieve $\tilde{O}(\sqrt{T})$ regret for the 0-1 loss, albeit very inefficiently.
\citet{AbernethyR09} posed an open problem: Is there an efficient bandit multiclass learning algorithm that achieves expected regret of
$\tilde{O}(\sqrt{T})$ with respect to any reasonable loss function?

The first attempt to solve this open problem was by \citet{CrammerG13}. Using a stochastic assumption about the mechanism generating the labels, they were able to show a $\tilde{O}(\sqrt{T})$ regret, with a second-order algorithm.

Later, \citet{HazanK11}, following a suggestion by \citet{AbernethyR09}, proposed the use of the log-loss coupled with a softmax prediction. The softmax depends on a parameter that controls the smoothing factor. The value of this parameter determines the exp-concavity of the loss, allowing \citet{HazanK11} to prove worst-case regret bounds that range between $O(\log T)$ and $O(T^\frac{2}{3})$, again with a second-order algorithm.
However, the choice of the smoothing factor in the loss becomes critical in obtaining strong bounds; see Appendix~\ref{sec:newtron} for detailed discussions.

The original Banditron algorithm has been also extended in many ways. \citet{WangJV10} have proposed a variant based on the exponentiated gradient algorithm~\citep{Kivinen-Warmuth-1997}. \citet{ValizadeganJW11} proposed different strategies to adapt the exploration rate to the data in the Banditron algorithm.
However, these algorithms suffer from the same theoretical shortcomings as the Banditron.

There has been significant recent focus on developing efficient algorithms for the general contextual bandit problem~\cite{DudikHKKLRZ11,Monster:2014,RakhlinS16,Syrgkanis2016a,Syrgkanis2016b}. While solving a more general problem that does not make assumptions on the structure of the reward vector or the policy class, these results assume that contexts or context/reward pairs are generated i.i.d., or the contexts to arrive are known beforehand, which we do not assume here.

In this paper, we follow a different route. Instead of designing an ad-hoc loss function that allows us to prove strong guarantees, we propose an algorithm that simultaneously satisfies a regret bound with respect to all the loss functions in a {family} of functions that are tight upper bounds to the 0-1 loss. The algorithm, named Second Order Banditron Algorithm (SOBA), is efficient and based on the second-order Perceptron algorithm~\cite{Cesa-BianchiCG05}. The regret bound is of the order of $\tilde O(\sqrt{T})$, providing a solution to the open problem of \citet{AbernethyR09}.

%% file: setting.tex
\section{Definitions and Settings}
\label{sec:setting}
We first introduce our notation.
Denote the rows of a matrix $\bm V \in \R^{k \times d}$ by $\bm v_1, \bm v_2, \ldots, \bm v_k$.  The vectorization of $\bm V$ is defined as $\vect(\bm V) = [\bm v_1, \bm v_2, \ldots, \bm v_k]^T$, which is a vector in $\R^{kd}$.
We define the reverse operation of reshaping a $kd\times 1$ vector into a $k\times d$ matrix by $\mat(\bm V)$, using a row-major order.
%the matricization operation $\mat(\cdot)$ by the matricization, that is rearranging a $kd \times 1$ vector to a $k \times d$ matrix by reversing the operations of vectorization.
To simplify notation, we will use $\bm V$ and $\vect(\bm V)$ interchangeably throughout the paper.
For matrices $\bm A$ and $\bm B$, denote by $\bm A \otimes \bm B$ their Kronecker product.
For matrices $\bm X$ and $\bm Y$ of the same dimension, denote by $\langle \bm X, \bm Y \rangle=\sum_{i,j} \bm X_{i,j} \bm Y_{i,j}$ their inner product. We use $\| \cdot \|$ to denote the $\ell_2$ norm of a vector, and $\|\cdot\|_F$ to denote the Frobenius norm of a matrix. For a positive definite matrix $\bm A$, we use $\| \bm x \|_{\bm A} = \sqrt{\inn{x}{\bm A x}}$ to denote the Mahalanobis norm of $\bm x$ with respect to $A$. We use $\one_k$ to denote the vector in $\R^k$ whose entries are all $1$s.

We use $\E_{t-1}[\cdot]$ to denote the conditional expectation given the observations
up to time $t-1$ and $\bm x_t, y_t$, that is, $\bm x_1$, $y_1$, $\tilde{y}_1$, \ldots,
 $\bm x_{t-1}$, $y_{t-1}$, $\tilde{y}_{t-1}$, $\bm x_t$, $y_t$.

Let $[k]$ denote $\{1,\ldots,k\}$, the set of possible labels.
In our setting, learning proceeds in rounds:

\noindent For $t = 1,2,\ldots,T:$
\begin{enumerate}
  \itemsep-0.0001cm
  \item The adversary presents an example $\bm x_t \in \R^d$ to the learner,
  and commits to a hidden label $y_t \in [k]$.
  \item The learner predicts a label $\tilde{y}_t \sim \bm p_t$,
  where $\bm p_t \in \Delta^{k-1}$ is a probability distribution over $[k]$.
  \item The learner receives the bandit feedback $\one[\tilde{y}_t \neq y_t]$.
\end{enumerate}

The goal of the learner is to minimize the total number of mistakes,
$M_T = \sum_{t=1}^T \one[\tilde{y}_t \neq y_t]$.

We will use linear predictors specified by a matrix $\bm W \in \R^{k\times d}$. The prediction is given by $\bm W(\bm x)=\argmax_{i\in[k]} (\bm W \bm x)_i$, where $(\bm W \bm x)_i$ is the $i$th element of the vector $\bm W \bm x$, corresponding to class $i$.

A useful notion to measure the performance of a competitor $\bm U\in \R^{k\times d}$ is the {\em multiclass hinge loss}
\begin{equation}
\label{eq:multi_hinge}
\ell(\bm U, (\bm x, y)) := \max_{i \neq y} [1 - (\bm U \bm x)_y + (\bm U \bm x)_i]_+,
\end{equation}
where $[\cdot]_+=\max(\cdot,0)$.

% Since the initial study of multiclass linear learning with bandit feedback~\cite{Kakade-Shalev-Shwartz-Tewari-2008},
% it has been a longstanding open question whether there exist some
% computationally efficient algorithm in multiclass bandit setting that has the following
% mistake bound:
% \[
% \sum_{t=1}^T M_t \leq \min_{\| U \|_F \leq G} \sum_{t=1}^T \ell(U, (x_t,y_t)) + O(k G D \sqrt{T})~.
% \]

% In this paper, we provide such a bound, with $\ell(\cdot, \cdot)$ replaced with a loss function
% at most the {\em squared hinge loss} of the optimal hypothesis. Our contributions
% are twofold:
% \begin{enumerate}
% \item We show a novel mistake bound on the classical (binary, multiclass) perceptron algorithm,
%       which depends on the cumulative squared hinge loss of the optimal linear predictor.
%       This bound is incomparable with the classical mistake bound that depends on the optimal hinge loss.
% \item We provide a novel unbiased estimator of the loss gradient in the online linear optimization
%       problem induced by multiclass linear prediction.
%       We further show that, when the estimator is plugged into a second order online linear optimization
%       algorithm, it yields a novel mistake bound.
% \end{enumerate}

%% file: loss.tex
\section{A History of Loss Functions}
\label{sec:loss}

As outlined in the introduction, a critical choice in obtaining strong theoretical guarantees is the choice of the loss function. In this section we introduce and motivate a family of multiclass loss functions.

In the full information setting, strong binary and multiclass mistake bounds are obtained through the use of the Perceptron algorithm~\citep{Rosenblatt58}. A common misunderstanding of the Perceptron algorithm is that it corresponds to a gradient descent procedure with respect to the (binary or multiclass) hinge loss. However, it is well known that the Perceptron simultaneously satisfies mistake bounds that depend on the cumulative hinge loss and also on the cumulative squared hinge loss, see for example~\citet{MohriR13}. Note also that the squared hinge loss is not dominated by the hinge loss, so, depending on the data, one loss can be better than the other.

We show that the Perceptron algorithm satisfies an even stronger mistake bound with respect to the cumulative loss of \emph{any} power of the multiclass hinge loss between 1 and 2.
\begin{theorem}
\label{thm:perc_bound}
On any sequence $(\bm x_1,y_1),\ldots,(\bm x_T,y_T)$ with
$\|\bm x_t\| \leq X$ for all $t\in [T]$,
and any linear predictor $\bm U \in \R^{k \times d}$,
the total number of mistakes $M_T$ of the multiclass Perceptron
satisfies, for any $q \in [1,2]$,
\[
M_T \leq M_T^{1-\frac{1}{q}} L_{\text{MH},q}^\frac{1}{q}(\bm U) + \|\bm U\|_F X \sqrt{2}\sqrt{M_T},
\]
where $L_{\text{MH},q}(\bm U) = \sum_{t=1}^T \ell(\bm W, (\bm x_t, y_t))^q$.
In particular, it simultaneously satisfies the following:
\[
M_T \leq L_{\text{MH},1}(\bm U)+2 X^2 \|\bm U\|^2_F + X \|\bm U\|_F \sqrt{2}\sqrt{L_{\text{MH},1}(\bm U)}
\]
\[
M_T \leq L_{\text{MH},2}(\bm U)+2 X^2 \|\bm U\|^2_F + X \|\bm U\|_F 2 \sqrt{2}\sqrt{L_{\text{MH},2}(\bm U)}~.
\]
\end{theorem}
For the proof, see Appendix~\ref{proof:thm:2}.

A similar observation was done by \citet{Orabona-Cesa-Bianchi-Gentile-2012} who proved a \emph{logarithmic} mistake bound with respect to all loss functions in a similar family of functions smoothly interpolating between the hinge loss to the squared hinge loss. In particular, \citet{Orabona-Cesa-Bianchi-Gentile-2012} introduced the following family of binary loss functions
\begin{equation}
\label{eq:binary_family}
\ell_\eta(x) := \begin{cases}
                1-\frac{2}{2-\eta}x+\frac{\eta}{2-\eta} x^2, & x \leq 1 \\
                0, & x > 1 \, .
              \end{cases}
\end{equation}
where $0\leq \eta \leq 1$. Note that $\eta=0$ recovers the binary hinge loss, and $\eta=1$ recovers the squared hinge loss.
Meanwhile, for any $0 \leq \eta \leq 1$, $\ell_\eta(x) \leq \max\{ \ell_0(x), \ell_1(x) \}$, and $\ell_\eta$ is an upper bound on 0-1 loss:
$\one[x < 0] \leq \ell_\eta(x)$.
See Figure~\ref{fig:loss_eta} for a plot of the different functions in the family.
%Note that for $\eta>0$, the loss $\ell_\eta(\cdot)$ has another positive part for $x>\frac{2-\eta}{\eta}$, hence the bound holds for all the $\eta$ that satisfy a constraint that depends on the norm of the competitor.

%It should also be noted that
%$\ell_\eta(x) \leq (1 - \frac{x}{2-\eta})^2$. We will show that under appropriate settings
%of $\eta$ in accordance with data properties, $\ell_\eta$ is at most \ell

\begin{figure}[t]
\centering
\includegraphics[width=0.45\textwidth]{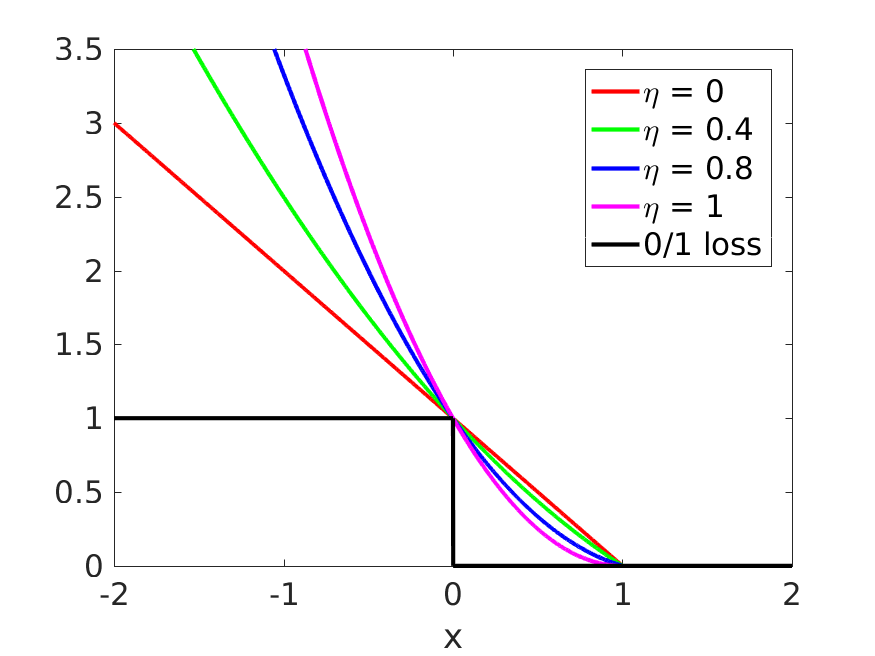}
\caption{Plot of the loss functions in $\ell_\eta$ for different values of $\eta$.}
\label{fig:loss_eta}
\end{figure}

Here, we define a multiclass version of the loss in \eqref{eq:binary_family} as
\begin{equation}
\label{eq:multi_family}
\ell_{\eta}(\bm U, (\bm x, y)) := \ell_\eta\left((\bm U \bm x)_y - \max_{i \neq y} (\bm U \bm x)_i \right).
\end{equation}
Hence, $\ell_{0}(\bm U, (\bm x, y))=\ell(\bm U, (\bm x, y))$ is the classical multiclass hinge loss and $\ell_{1}(\bm U, (\bm x, y))=\ell^2(\bm U, (\bm x, y))$ is the squared multiclass hinge loss.

Our algorithm has a $\tilde O(\frac{1}{\eta}\sqrt{T})$ regret bound that holds simultaneously for all loss functions in this family, with $\eta$ in a range that ensure that  $(\bm U \bm x)_i - (\bm U \bm x)_j \leq \frac{2-\eta}{\eta}$, $i ,j \in [k]$.
%Hence, unlike in \citet{HazanK11}, the algorithm does not need to know the smoothness of the loss function.
We also show that there exists a setting of the parameters of the algorithm that gives a mistake upper bound of $\tilde O((L^*T)^{1/3} + \sqrt{T})$, where $L^*$ is the cumulative hinge loss of the competitor, which is never worse that the best bounds in \citet{Kakade-Shalev-Shwartz-Tewari-2008}.

%% file: result.tex
\section{Second Order Banditron Algorithm}
\label{sec:main_result}

\begin{algorithm}[t]
\caption{Second Order Banditron Algorithm (SOBA)}
\label{alg:soba}
\begin{algorithmic}[1]
\REQUIRE Regularization parameter $a > 0$, exploration parameter $\gamma \in [0,1]$.
\STATE Initialization: $\bm W_1 = \bm 0$, $\bm A_0 = a \bm I$, $\bm \theta_0 = \bm 0$
\FOR{$t = 1,2,\ldots,T$}
\STATE Receive instance $\bm x_t \in \R^d$
\STATE $\hat{y}_t = \arg\max_{i \in [k]} (\bm W_t \bm x_t)_i$
\STATE Define $\bm p_t = (1-\gamma) \bm e_{\hat{y}_t} + \frac \gamma k \one_k$
%$p_{t,i} = (1-\gamma)I[i=\hat{y}_t]+\tfrac{\gamma}{k}$ for $i \in [k]$
\STATE Randomly sample $\tilde{y}_t$ according to $\bm p_t$
\STATE Receive bandit feedback $\one[\tilde{y}_t \not= y_t]$
\STATE Initialize update indicator $n_t = 0$
\IF{$\tilde{y}_t = y_t$}
\STATE $\bar{y}_t = \arg\max_{i \in [k] \setminus \{y_t\}} (\bm W_t \bm x_t)_i$
\STATE $\bm g_t = \frac{1}{p_{t,y_t}} (\bm e_{\bar{y}_t} - \bm e_{y_t}) \otimes \bm x_t$
\STATE $\bm z_t = \sqrt{p_{t,y_t}} \bm g_t$
\STATE $m_t = \frac{\inn{\bm W_t}{\bm z_t}^2 + 2 \inn{\bm W_t}{\bm g_t}}{1+\bm z_t^T \bm A^{-1}_{t-1} \bm z_t}$
\IF{$m_t + \sum_{s = 1}^{t-1} n_s m_s \geq 0$}
       \label{alg:soba:updatecriterion}
       \STATE Turn on update indicator $n_t = 1$
\ENDIF
%\ELSE
%\STATE $n_t = 0$, $m_t = 0$
%\STATE $\bm g_t = \frac{I[\tilde{y}_t = y_t]}{p_{t,y_t}} (e_{\hat{y}_t} - e_{y_t}) \otimes x_t$
%\STATE $\bm z_t = \sqrt{p_{t,y_t}} \bm g_t$
\ENDIF
\STATE Update $\bm A_t = \bm A_{t-1} + n_t \bm z_t \bm z_t^T$ \label{alg:soba:updateA}
\STATE Update $\bm \theta_t = \bm \theta_{t-1} - n_t \bm g_t$ \label{alg:soba:updatetheta}
\STATE Set $\bm W_{t+1} = \mat(\bm A_t^{-1} \bm \theta_t)$ \label{alg:soba:updateW}
%\STATE $\bm A_t = a \bm I + \sum_{s = 1}^t n_s \bm z_s \bm z_s^T$
%\STATE $\bm W_{t+1} \gets \mat(\bm A_{t}^{-1} (-\sum_{s = 1}^t n_s \bm g_s))$
\ENDFOR
\end{algorithmic}
\textbf{Remark:} matrix $\bm A_t$ is of dimension $kd \times kd$, and vector $\bm \theta_t$ is of dimension $kd$;
in line 20, the matrix multiplication results in a $kd$ dimensional vector,
which is reshaped to matrix $\bm W_{t+1}$ of dimension $k \times d$.
\end{algorithm}

This section introduces our algorithm for bandit multiclass online learning, called Second Order Banditron Algorithm (SOBA), described in Algorithm~\ref{alg:soba}. In Appendix~\ref{sec:expconcave}, we introduce a conceptually simpler version of SOBA (Algorithm~\ref{alg:soba-mod}); we defer the comparison of the two algorithms therein. 

SOBA makes a prediction using the $\gamma$-greedy strategy: At each iteration $t$,
with probability $1-\gamma$,
it predicts $\hat{y}_t = \argmax_{i \in [k]} (\bm W_t \bm x_t)_i$;
with the remaining probability $\gamma$, it selects a random action in $[k]$.
As discussed in~\citet{Kakade-Shalev-Shwartz-Tewari-2008},
randomization is essential for designing bandit multiclass learning
algorithms. If we deterministically
output a label and make a mistake, then it is hard to make an update
since we do not know the identity of $y_t$.
However, if randomization is used, we can
\emph{estimate} $y_t$ and perform online stochastic mirror descent type updates~\citep{Bubeck-Cesa-Bianchi-2012}.

SOBA keeps track of two model parameters: cumulative Perceptron-style updates
$\bm \theta_t = - \sum_{s=1}^t n_s \bm g_s \in \R^{kd}$
and corrected covariance matrix $\bm A_t = a\bm I + \sum_{s=1}^t n_s \bm z_s \bm z_s^T \in \R^{kd \times kd}$.
The classifier $\bm W_t$ is computed by matricizing over the matrix-vector product
 $\bm A_{t-1}^{-1} \bm \theta_{t-1} \in \R^{k d}$.
 %to get a $k \times d$ matrix.
The weight vector $\bm \theta_t$ is standard in designing online mirror descent type
algorithms~\citep{Shalev-Shwartz-2011, Bubeck-Cesa-Bianchi-2012}.
The matrix $\bm A_t$
is standard in designing online learning algorithms with adaptive regularization~\citep{Cesa-BianchiCG05, CrammerK09, McMahan-Streeter-2010,DuchiHS11,Orabona-Crammer-Cesa-Bianchi-2015}.
The algorithm updates its model ($n_t = 1$) only when the following conditions hold simultaneously:
(1) the predicted label is correct ($\tilde{y}_t = y_t$), and
(2) the ``cumulative regularized negative margin'' ($\sum_{s=1}^{t-1} n_s m_s + m_t$) is positive if this update were performed.
Note that when the predicted label is correct we know the identity of the true label.

%SOBA uses an indicator $n_t$ to denote if there is an update at iteration $t$.
As we shall see, the set of iterations where $n_t = 1$ includes all iterations where
$\tilde{y}_t = y_t \neq \hat y_t$. This fact is crucial to the mistake bound analysis.
Furthermore, there are some iterations where $\tilde{y}_t = y_t = \hat y_t$ but
we still make an update.
This idea is related to ``online passive-aggressive algorithms''~\citep{CrammerD06, CrammerK09} in the full information setting, where the algorithm makes an update even when it predicts correctly but the margin is too small.

Let's now describe our algorithm more in details.
%In order to construct a second order algorithm for the bandit multiclass problem, we need a unbiased estimate of the correct update and
%a specific time-varying elliptic potential.
Throughout, suppose all the examples are $\ell_2$-bounded: $\|\bm x_t\|_2 \leq X$.

%Our first contribution is in proposing a new unbiased estimator of the gradient of the multiclass hinge loss as
%\[
%\bm g_t = \frac{I[y_t = \tilde{y}_t]}{p_{t,y_t}} x_t \otimes (e_{\hat{y}_t} - e_{y_t})~.\]
%Note that this estimator has the nice property that it is nonzero only if
%$y_t = \tilde{y}_t \neq \hat{y}_t$. That is, it only makes an update when the
%algorithm is exploring and the exploration label $\tilde{y}_t$ ``luckily'' hits
%the true label $y_t$.
%When $M_t = 0$, that is $\tilde{y}_t = y_t$, we are able to observe the

As outlined above, we associate a time-varying regularizer $R_t(\bm W) = \frac12\| \bm W \|_{\bm A_t}^2$,
where $\bm A_t = a \bm I + \sum_{s=1}^t n_s \bm z_s \bm z_s^T$ is a $kd \times kd$ matrix and
\[
\bm z_t = \sqrt{p_{t,y_t}} \bm g_t = \frac{1}{\sqrt{p_{t,y_t}}} (\bm e_{\bar{y}_t} - \bm e_{y_t}) \otimes \bm x_t~.
\]
Note that this time-varying regularizer is constructed by scaled versions of the updates $\bm g_t$. This is critical, because in expectation this becomes the correct regularizer.
%Indeed, when $\bm g_t\neq \boldsymbol{0}$ we know that $p_{t,y_t} = \frac{\gamma}{k}$
%holds determinstically. Therefore, we can rewrite $\bm g_t$ and $\bm z_t$ as
%\begin{align*}
%\bm g_t &= \frac{k I[y_t = \tilde{y}_t \neq \hat{y}_t]}{\gamma} (\bm e_{y_t} - \bm e_{\hat{y}_t}) \otimes \bm x_t,\\
%\bm z_t &= \sqrt{\frac{\gamma}{k}} \bm g_t
%    = \sqrt{\frac{k I[y_t = \tilde{y}_t \neq \hat{y}_t]}{\gamma}} (e_{y_t} - e_{\hat{y}_t}) \otimes x_t~.
%\end{align*}
Indeed, it is easy to verify that, for any $\bm U \in \R^{k \times d}$,
\begin{align*}
&\E_{t-1} [\one[y_t = \tilde{y}_t] \, \bm g_t] = (\bm e_{\bar{y}_t} - \bm e_{y_t}) \otimes \bm x_t, \\
%\label{eq:expectation_g}
&\E_{t-1} [\one[y_t = \tilde{y}_t] \, \inn{\bm U}{\bm z_t}^2] = \inn{\bm U}{(\bm e_{\bar{y}_t} - \bm e_{y_t}) \otimes \bm x_t}^2. %\label{eq:expectation_z2}
\end{align*}
In words, this means that in expectation the regularizer contains the outer products of the updates, that in turn promote the correct class and demotes the wrong one.
We stress that it is impossible to get the same result with the estimator proposed in \citet{Kakade-Shalev-Shwartz-Tewari-2008}.
Also, the analysis is substantially different from the Online Newton Step approach~\citep{Hazan-Agarwal-Kale-2007} used in \citet{HazanK11}.

In reality, we do not make an update in all iterations in which $\tilde{y}_t = y_t$,
since the algorithm need to maintain the invariant that $\sum_{s=1}^t m_s n_s \geq 0$,
which is crucial to the proof of Lemma~\ref{lem:vaw}.
Instead, we prove a technical lemma that gives an explicit form on the expected
update $n_t \bm g_t$ and expected regularization $n_t \bm z_t \bm z_t^T$.
%This turns out to be useful in the proof of Theorem~\ref{thm:yhat} below.
Define
%\[ q_t := \one\left[\sum_{s=1}^{t-1} n_s m_s + \frac{\inn{\bm W_t}{\bm z_t}^2 + 2\inn{\bm W_t}{\bm g_t}}{1 + \bm z_t^T \bm A_{t-1}^{-1} \bm z_t} \geq 0\right], \]
\begin{align*}
q_t &:= \one\left[\sum_{s=1}^{t-1} n_s m_s + m_t \geq 0\right],\\
h_t &:= \one[\hat{y}_t \neq y_t] + q_t \one[\hat{y}_t = y_t]~.
\end{align*}
\begin{lemma}
  For any $\bm U \in \R^{k d}$,
  \[ \E_{t-1} \left[ n_t \inn{\bm U}{\bm g_t} \right] = h_t \inn{\bm U}{(\bm e_{y_t} - \bm e_{\bar{y}_t}) \otimes \bm x_t}, \]
  \[ \E_{t-1} \left[ n_t \inn{\bm U}{\bm z_t}^2 \right] = h_t \inn{\bm U}{(\bm e_{y_t} - \bm e_{\bar{y}_t}) \otimes \bm x_t}^2. \]
     %$F_t \one[\hat{y}_t = y_t]$ is a function of $\bm x_1$, $y_1$, $\tilde{y}_1$, \ldots,
     %$\bm x_{t-1}$, $y_{t-1}$, $\tilde{y}_{t-1}$, $\bm x_t$, $y_t$. Furthermore,
     %\begin{align*}
      % \E_{t-1} &\left[ n_t \left(2 \inn{\bm U}{-\bm g_t} - \inn{\bm U}{\bm z_t}^2 \right) \right] \\
      % &=   2 h_t \inn{\bm U}{(\bm e_{y_t} - \bm e_{\bar{y}_t}) \otimes \bm x_t} \\
      % &\quad - h_t \inn{\bm U}{(\bm e_{y_t} - \bm e_{\bar{y}_t}) \otimes \bm x_t}^2 .
     %\end{align*}
\label{lem:nt}
\end{lemma}
The proof of Lemma~\ref{lem:nt} is deferred to the end of Subsection~\ref{sec:proof_main}.

Our last contribution is to show how our second order algorithm satisfies a mistake bound for an entire family of loss functions. Finally, we relate the performance of the algorithm that predicts $\hat{y}_t$ to the $\gamma$-greedy algorithm.

Putting all together, we have our expected mistake bound for SOBA.~\footnote{Throughout the paper, expectations are taken with respect to the randomization of the algorithm.}
\begin{theorem}
\label{thm:soba_bound}
SOBA has the following expected upper bound on the number of mistakes, $M_T$, for any $\bm U \in \R^{k \times d}$
and any $0 < \eta \leq \min(1, \frac{2}{2\max_i \|\bm u_i\| X + 1})$,
\begin{align*}
\E\left[ M_T \right]
&\leq L_\eta(\bm U)+\frac{a\eta}{2-\eta}\| \bm U \|_F^2 \\
&\quad +\frac{k}{\gamma \eta(2-\eta) } \sum_{t=1}^T \E\left[\bm z_t^T \bm A_t^{-1} \bm z_t\right]+\gamma T,
\end{align*}
where $L_\eta(\bm U) := \sum_{t=1}^T \ell_{\eta}(\bm U, (\bm x_t, y_t))$
%$L_\eta(\bm U) := \sum_{t=1}^T \ell_{\eta}(\max_{i \neq y_t}\inn{\bm U}{(\bm e_{y_t} - \bm e_i) \otimes \bm x_t})$
is the cumulative $\eta$-loss of the linear predictor $\bm U$, and $\{\bm u_i\}_{i=1}^k$ are rows of $\bm U$.

In particular, setting $\gamma=O( \sqrt{\frac{k^2\,d \ln T}{T}})$ and $a = X^2$, we have
\[
\E\left[ M_T \right]
\leq L_\eta(\bm U)+O\left(X^2 \|\bm U\|^2_F+\frac{k}{\eta}\sqrt{d T \ln T}\right)~.
\]
\end{theorem}
Note that, differently from previous analyses~\citep{Kakade-Shalev-Shwartz-Tewari-2008,CrammerG13,HazanK11}, we do not need to assume a bound on the norm of the competitor, as in the full information Perceptron and Second Order Perceptron algorithms.
In Appendix~\ref{sec:adaptive}, we also present an adaptive variant of SOBA that sets exploration rate $\gamma_t$ dynamically, which achieves a regret bound within a constant factor of that using optimal tuning of $\gamma$.

We prove Theorem~\ref{thm:soba_bound} in the next Subsection, while in Subsection~\ref{sec:fallback} we prove a mistake bound with respect to the hinge loss, that is not fully covered by Theorem~\ref{thm:soba_bound}.

%% file: proof_main.tex
\subsection{Proof of Theorem~\ref{thm:soba_bound}}
\label{sec:proof_main}

Throughout the proofs,
$\bm U$, $\bm W_t$, $\bm g_t$, and $\bm z_t$'s should be thought of as $kd \times 1$ vectors.
We first show the following lemma. Note that this is a statement over any sequence and no expectation is taken.
%, and $\bm A_t$'s are $kd \times kd$ matrices.
%
\begin{lemma}
\label{lem:vaw}
For any $\bm U \in \R^{kd}$, with the notation of Algorithm~\ref{alg:soba}, we have:
\begin{align*}
\sum_{t=1}^T &n_t \left(2 \inn{\bm U}{-\bm g_t} - \inn{\bm U}{\bm z_t}^2 \right) \\
&\leq a  \| \bm U\|_F^2
%+ \frac{d k^2}{\gamma} \ln\left(1 + \frac{2X^2 \sum_{i=1}^T I[y_t = \tilde{y}_t \neq \hat{y}_t]}{a\, \gamma\, d}\right)~.
+ \sum_{t=1}^T n_t \bm g_t^T \bm A_t^{-1} \bm g_t~.
\end{align*}
\end{lemma}
\begin{proof}
First, from line 14 of Algorithm~\ref{alg:soba}, it can
be seen (by induction) that
SOBA maintains the invariant that
\begin{equation}
  \sum_{s=1}^t n_s m_s \geq 0.
  \label{eqn:invariant}
\end{equation}

%~\ref{alg:soba:updatecriterion}
%Following~\cite{Orabona-Cesa-Bianchi-Gentile-2012},
We next reduce the proof to the regret analysis of online least squares problem.
%Define $\bm U = \sqrt{\frac{k}{\gamma}}$.
For iterations where $n_t = 1$, define $\alpha_t = \frac{1}{\sqrt{p_{t,y_t}}}$ so that
$\bm g_t = \alpha_t \bm z_t$. From the algorithm, $\bm A_t = aI + \sum_{s=1}^t n_s \bm z_s \bm z_s^T$, and
$\bm W_t$ is the ridge regression solution based on data collected in time $1$ to $t-1$,
i.e. $\bm W_t = \bm A_{t-1}^{-1} (-\sum_{s=1}^{t-1} n_s \bm g_s) = \bm A_{t-1}^{-1} (-\sum_{s=1}^{t-1} n_s \alpha_s \bm z_s)$.

By per-step analysis in online least squares,~\citep[see, e.g.,][]{Orabona-Cesa-Bianchi-Gentile-2012}(See Lemma~\ref{lem:perstep} for a proof), we have that if an update is made at iteration $t$, i.e. $n_t = 1$, then
\begin{align*}
  &\frac12(\inn{\bm W_t}{\bm z_t} + \alpha_t)^2 (1 - \bm z_t^T \bm A_t^{-1} \bm z_t)
  - \frac12(\inn{\bm U}{\bm z_t} + \alpha_t)^2 \\
  &\quad \leq
  \frac 1 2 \| \bm U - \bm W_t \|_{\bm A_{t-1}}^2 - \frac 1 2 \| \bm U - \bm W_{t+1} \|_{\bm A_t}^2~.
\end{align*}
Otherwise $n_t = 0$, in which case we have $\bm W_{t+1} = \bm W_t$ and $\bm A_{t+1} = \bm A_t$.

Denoting by $k_t=1-\bm z_t^T \bm A_t^{-1} \bm z_t$, by Sherman-Morrison formula, $k_t=\frac{1}{1+\bm z_t^T \bm A_{t-1}^{-1} \bm z_t}$. Summing over all rounds $t \in [T]$ such that $n_t = 1$,
\begin{align*}
   &\frac12 \sum_{t=1}^T n_t \left[ (\inn{\bm W_t}{\bm z_t} + \alpha_t)^2 k_t - (\inn{\bm U}{\bm z_t} + \alpha_t)^2 \right] \\
   &\quad \leq \frac 1 2 \| \bm U \|_{\bm A_0}^2 - \frac 1 2 \| \bm U - \bm W_{T+1} \|_{\bm A_T}^2\leq  \frac a 2 \| \bm U \|_F^2~.
\end{align*}
%We now distinguish two cases. In the case that $\hat{y}_t=\tilde{y}_t=y_t$, we perform a margin update and
We also have by definition of $m_t$,
\begin{align*}
&(\inn{\bm W_t}{\bm z_t} + \alpha_t)^2 k_t - (\inn{\bm U}{\bm z_t} + \alpha_t)^2 \\
%&\quad = \inn{\bm W_t}{\bm z_t}^2+2 k_t \inn{\bm W_t}{\bm g_t} - 2 \inn{\bm U}{\bm g_t} - \inn{\bm U}{\bm z_t}^2 -\alpha_t^2 \bm z_t^T \bm A_t^{-1} \bm z_t\\
&\quad =  m_t - 2 \inn{\bm U}{\bm g_t} - \inn{\bm U}{\bm z_t}^2 -\alpha_t^2 \bm z_t^T \bm A_t^{-1} \bm z_t~.
\end{align*}
%On the other hand, when $\hat{y}_t \neq \tilde{y}_t = y_t$, we have
%\begin{align*}
%&(\inn{\bm W_t}{\bm z_t} + \alpha_t)^2 k_t - (\inn{\bm U}{\bm z_t} + \alpha_t)^2 \\
%&\quad \geq  - 2 \inn{\bm U}{\bm g_t} - \inn{\bm U}{\bm z_t}^2 -\alpha_t^2 \bm z_t^T \bm A_t^{-1} \bm z_t,
%\end{align*}
%because $\inn{\bm W_t}{\bm z_t}\geq 0$ and $k_t \geq 0$.
Putting all together and using the fact that $\sum_{t=1}^T n_t m_t\geq 0$, we have the stated bound.
\end{proof}
\vspace{-0.2cm}
We can now prove the following mistake bound for the prediction $\hat{y}_t$, defined as
$\hat{M}_T := \sum_{t=1}^T \one[\hat{y}_t \neq y_t]$.
\begin{theorem}
For any $\bm U \in \R^{k \times d}$,
and any $0 < \eta \leq \min(1, \frac{2}{2\max_i \|\bm u_i\| X + 1})$, the expected number of mistakes committed by $\hat{y}_t$ can be bounded as
\begin{align*}
\E\left[\hat{M}_T \right]
&\leq L_\eta(\bm U)+\frac{a\eta \| \bm U \|_F^2}{2-\eta}
 + \frac{k \sum_{t=1}^T \E[\bm n_t z_t^T \bm A_t^{-1} \bm z_t]}{\gamma \eta(2-\eta) } \\
&\leq L_\eta(\bm U)+\frac{a\eta \| \bm U \|_F^2}{2-\eta}
 + \frac{d k^2 \ln\left(1 + \frac{2T\,X^2}{a\,d\,k}\right)}{\gamma \eta(2-\eta) },
\end{align*}
where $L_\eta(\bm U) := \sum_{t=1}^T \ell_\eta(\bm U, (\bm x_t, y_t))$
%(\max_{i \neq y_t}\inn{\bm U}{(\bm e_{y_t} - \bm e_i) \otimes \bm x_t})
is the $\eta$-loss of the linear predictor $\bm U$.
\label{thm:yhat}
\end{theorem}
\begin{proof}
Using Lemma~\ref{lem:vaw} with $\eta \bm U$, we get that
\begin{align*}
\sum_{t=1}^T n_t \left(2\eta \inn{\bm U}{-\bm g_t} - \eta^2 \inn{\bm U}{\bm z_t}^2 \right) \\
\leq a \eta^2 \| \bm U\|_F^2
+ \sum_{t=1}^T n_t \bm g_t^T \bm A_t^{-1} \bm g_t~.
\end{align*}

Taking expectations, using Lemma~\ref{lem:nt} and the fact that $\frac{1}{p_{t,y_t}} \leq \frac{k}{\gamma}$ and that $A_t$ is positive definite, we have
\begin{equation}
\label{eq:thm_haty_eq1}
\begin{split}
0 & \leq -\E \left[\sum_{t=1}^T h_t \cdot 2 \eta \inn{\bm U}{(\bm e_{y_t} - \bm e_{\bar{y}_t}) \otimes \bm x_t}\right] \\
& + \E \left[\sum_{t=1}^T h_t \cdot \eta^2 (\inn{\bm U}{(\bm e_{y_t} - \bm e_{\bar{y}_t}) \otimes \bm x_t})^2\right] \\
&\quad + a  \eta^2 \| \bm U \|_F^2+ \frac{k} {\gamma} \sum_{t=1}^T \E[n_t \bm z_t^T \bm A_t^{-1} \bm z_t]~.
\end{split}
\end{equation}
Add the terms
$\eta(2-\eta) \E\left[\sum_{t=1}^T h_t \right]$ to both sides and divide both sides by $\eta(2-\eta)$, to have
\begin{align*}
\E&\left[\sum_{t=1}^T h_t \right]
\leq \E\left[\sum_{t=1}^T h_t f(\inn{\bm U}{(\bm e_{y_t} - \bm e_{\bar{y}_t}) \otimes \bm x_t})\right] \\
&\quad +\frac{a\eta}{2-\eta} \| \bm U \|_F^2+\frac{k}{\gamma \eta(2-\eta)} \sum_{t=1}^T \E[n_t \bm z_t^T \bm A_t^{-1} \bm z_t],
\end{align*}
where $f(z) := 1 - \frac{2}{2-\eta} z + \frac{\eta}{2-\eta} z^2$.
Taking a close look at the function $f$, we observe that the two roots of the quadratic function  are $1$ and $\frac{2-\eta}{\eta}$,
respectively. Setting $\eta\leq 1$, the function is negative in $(1, \frac{2-\eta}{\eta}]$ and positive in $(-\infty, 1]$.
Additionally, if $0 < \eta \leq \frac{2}{2\max_i \|\bm u_i\|_2 X + 1}$, then for all $i,j \in [k]$, $\inn{\bm U}{(\bm e_i - \bm e_j) \otimes \bm x_t} \leq \frac{2-\eta}{\eta}$. Therefore, we have that
\begin{align*}
&f(\inn{\bm U}{(\bm e_{y_t} - \bm e_{\bar{y}_t}) \otimes \bm x_t}) \\
&\quad = f( (\bm U \bm x_t)_{y_t} - (\bm U \bm x_t)_{\bar y_t} ) \\
&\quad \leq \ell_\eta\left( (\bm U \bm x_t)_{y_t} - (\bm U \bm x_t)_{\bar y_t} \right) \\
&\quad \leq \ell_\eta\left( (\bm U \bm x_t)_{y_t} - \max_{r \neq y_t}(\bm U \bm x_t)_r \right)
=\ell_\eta(\bm U, (\bm x_t,y_t))~.
\end{align*}
where the first equality is from algebra, the first inequality is from that
$f(\cdot) \leq \ell_\eta(\cdot)$ in $(-\infty, \frac{2-\eta}{\eta}]$,
the second inequality is from
that $\ell_\eta(\cdot)$ is monotonically decreasing.

Putting together the two constraints on $\eta$, and noting that $\hat{M}_T \leq \sum_{t=1}^T h_t$, we have the first bound.

The second statement follows from Lemma~\ref{lem:adaptivenorms} below.
\end{proof}

\begin{lemma}
If $d \geq 1$, $k \geq 2$, $T \geq 2$, then
\[ \sum_{t=1}^T \E[n_t \bm z_t^T \bm A_t^{-1} \bm z_t] \leq d k \ln\left(1 + \frac{2X^2 T}{a\,d \, k}\right)~. \]
Specifically, if $a = X^2$, the right hand side is $ \leq d k \ln T$.
\label{lem:adaptivenorms}
\end{lemma}
\begin{proof}
Observe that
\begin{align*}
&\sum_{t=1}^T n_t \bm z_t^T \bm A_t^{-1} \bm z_t
\leq \ln \frac{|\bm A_T|}{|\bm A_0|} \\
&\quad \leq d \, k \ln\left(1+\frac{2 X^2 \sum_{t=1}^T \frac{\one[\tilde{y}_t = y_t ]}{p_{t,y_t}}}{a \, d \, k}\right),
\end{align*}
where the first inequality is a well-known fact from linear algebra~\citep[e.g.][Lemma 11]{Hazan-Agarwal-Kale-2007}.
Given that the $\bm A_T$ is $kd \times kd$, the second inequality comes from the fact that $|\bm A_T|$ is maximized when all its eigenvalues are equal to
$\frac{\tr(\bm A_T)}{d \, k} = a+\frac{\sum_{t=1}^T n_t \| \bm z_t \|^2}{d \, k} \leq a+\frac{2 X^2 \sum_{t=1}^T \frac{\one[\tilde{y}_t = y_t ]}{p_{t,y_t}}}{d \, k}$.
Finally, using Jensen's inequality, we have that,
\[
\sum_{t=1}^T \E[\bm n_t z_t^T \bm A_t^{-1} \bm z_t] \leq d \, k \ln\left(1 + \frac{2X^2 T}{a\,d \, k}\right)~.
\]
If $a = X^2$, then the right hand side is $d \, k \ln(1 + \frac{2T}{d \, k})$,
 which is at most $dk \ln T$ under the conditions on $d$, $k$, $T$.
\end{proof}

\begin{proof}[Proof of Theorem~\ref{thm:soba_bound}]
%Note that when playing $\gamma$-greedy exploration, our algorithm is
%at most $T\gamma$-worse than the performance of $\hat{y}_t$.
%Using this observation, we arrive at the result stated in Theorem~\ref{thm:soba_bound}.
Observe that by triangle inequality, $\one[\tilde y_t \neq y_t] \leq \one[\tilde y_t \neq \hat{y}_t] + \one[y_t \neq \hat{y}_t]$.
Summing over $t$, taking expectation on both sides, we conclude that
\begin{equation}
  \E[M_T] \leq \E[\hat{M}_T] + \gamma T~.
  \label{eqn:greedy}
\end{equation}
The first statement follows from combining the above inequality with Theorem~\ref{thm:yhat}.

For the second statement, first note that from Theorem~\ref{thm:yhat},
and Equation~\eqref{eqn:greedy}, we have
\begin{align*}
  \E[M_T] &\leq L_\eta(\bm U)+\frac{a\eta \| \bm U \|_F^2}{2-\eta}  + \frac{d k^2 \ln\left(1 + \frac{2T\,X^2}{a\,d\,k}\right)}{\gamma \eta(2-\eta) }  + \gamma T\\
&\leq  L_\eta(\bm U)+ X^2 \| \bm U \|_F^2  + \frac{2 d\, k^2 \, \ln T}{\gamma \eta }  + \gamma T,
\end{align*}
where the second inequality is from that $\eta \leq 1$, and
Lemma~\ref{lem:adaptivenorms} with $a = X^2$. The statement is concluded
by the setting of $\gamma = O(\sqrt{\frac{k^2d\ln T}{T}})$.
\end{proof}

\begin{proof}[Proof of Lemma~\ref{lem:nt}]
We show the lemma in two steps.
Let $G_t := q_t \cdot \one[y_t = \tilde{y}_t = \hat{y}_t]$, and
$H_t := \one[y_t = \tilde{y}_t \neq \hat{y}_t]$.

First, we show that $n_t = G_t + H_t$. Recall that SOBA maintains the invariant~\eqref{eqn:invariant}, hence
$\sum_{s=1}^{t-1} n_s m_s \geq 0$.
From line 14 of SOBA, we see that $n_t = 1$ only if $\tilde{y}_t = y_t$. Now consider
two cases:
\begin{itemize}
  \item $y_t = \tilde{y}_t \neq \hat{y}_t$. In this case, $\bar{y}_t = \hat{y}_t$, therefore
  $\inn{\bm W_t}{\bm g_t} \geq 0$, making $m_t \geq 0$. This implies that $\sum_{s=1}^{t-1} n_s m_s + m_t \geq 0$,
  guaranteeing $n_t = 1$.

  \item $y_t = \tilde{y}_t = \hat{y}_t$. In this case, $n_t$ is set to $1$ if and only
  if $q_t = 1$, i.e. $\sum_{s=1}^{t-1} n_s m_s + m_t \geq 0$.
\end{itemize}
This gives that $n_t = G_t + H_t$.

Second, we have the following two equalities:
 \begin{align*}
   &\E_{t-1} \left[ H_t \inn{\bm U}{\bm g_t} \right] \\
   &= \E_{t-1} \left[ \frac{\one[\tilde{y}_t = y_t]}{p_{t,y_t}} \one[\hat{y}_t \neq y_t] \inn{\bm U}{(\bm e_{y_t} - \bm e_{\bar{y}_t}) \otimes \bm x_t} \right] \\
   &= \one[\hat{y}_t \neq y_t] \inn{\bm U}{(\bm e_{y_t} - \bm e_{\bar{y}_t}) \otimes \bm x_t},
 \end{align*}
 \begin{align*}
   &\E_{t-1} \left[ G_t \inn{\bm U}{\bm g_t} \right] \\
   &= \E_{t-1} \left[ \frac{\one[\tilde{y}_t = y_t]}{p_{t,y_t}}  \one[\hat{y}_t = y_t]  q_t  \inn{\bm U}{(\bm e_{y_t} - \bm e_{\bar{y}_t}) \otimes \bm x_t} \right] \\
   &= \one[\hat{y}_t \neq y_t] q_t \inn{\bm U}{(\bm e_{y_t} - \bm e_{\bar{y}_t}) \otimes \bm x_t}~.
 \end{align*}
The first statement follows from adding up the two equalities above.

The proof for the second statement is identical, except replacing $\inn{\bm U}{(\bm e_{y_t} - \bm e_{\bar{y}_t}) \otimes \bm x_t}$ with $\inn{\bm U}{(\bm e_{y_t} - \bm e_{\bar{y}_t}) \otimes \bm x_t}^2$.
%using Equation~\eqref{eq:expectation_z2} as opposed to Equation~\eqref{eq:expectation_g}.
\end{proof}

%% file: fallback.tex
\subsection{Fall-Back Analysis}
\label{sec:fallback}

The loss function $\ell_{\eta}$ is an interpolation between the hinge and the
squared hinge losses. Yet, the bound becomes vacuous for $\eta=0$.
Hence, in this section we show that SOBA also guarantees a $\tilde O((L_0(\bm U)T)^{1/3} + \sqrt{T})$ mistake bound w.r.t. $L_0(\bm U)$, the multiclass hinge loss of the competitor, assuming $L_0(\bm U)$ is known.
Thus the algorithm achieves a mistake guarantee no worse than the sharpest
bound implicit in~\citet{Kakade-Shalev-Shwartz-Tewari-2008}.

\begin{theorem}
Set $a=X^2$ and denote by $M_T$ the number of mistakes done by SOBA. Then SOBA has the following guarantees:\footnote{Assuming the knowledge of $\|\bm U\|_F$ it would be possible to reduce the dependency on $\|\bm U\|_F$ in both bounds. However such assumption is extremely unrealistic and we prefer not to pursue it.}
\begin{enumerate}
\item If $L_0(\bm U) \geq (\|\bm U\|_F^2 + 1) \sqrt{ d k^2 X^2 T \ln T}$, then with parameter setting
$\gamma = \min(1,(\frac{d k^2 X^2 L_0(\bm U) \ln T}{T^2})^{1/3})$, one has the following expected mistake bound:
\begin{align*}
\E [&M_T] \leq L_0(\bm U)\\
& + O\Big(\|\bm U\|_F(d\, k^2 X^2 L_0(\bm U) T \ln T)^{1/3} \Big)~.
\end{align*}
\item If $L_0(\bm U) < (\|\bm U\|_F^2 + 1) \sqrt{ d k^2 X^2 T \ln T}$, then with parameter setting
$\gamma = \min(1,(\frac{d\,k^2 X^2 \ln T}{T})^{1/2})$,
one has the following expected mistake bound:
\[
\E [M_T] \leq L_0(\bm U) + O\left(k (\|\bm U\|_F^2 + 1) X \sqrt{d T \ln T}\right)~.
\]
\end{enumerate}
where $L_0(\bm U) := \sum_{t=1}^T \ell_0(\bm U, (\bm x_t, y_t))$ is the hinge loss of the linear classifier $\bm U$.
\end{theorem}

% \paragraph{Remark.} In the theorem we measure the hinge loss of $\bm U$ with repsect to margin $1$.
% A mistake bound with respect to $\zeta$-hinge loss of $\bm U$
% can be easily obtained by rescaling $\bm U$. Consequently,
% the $d_m$ in the new bound should be changed to $\frac{D^2X^2}{\zeta^2}$.

\begin{proof}
Recall that $\hat{M}_T$ the mistakes made by $\hat y_t$, that is $\sum_{t=1}^T \one[\hat{y}_t \neq y_t]$.
Adding to both sides of~\eqref{eq:thm_haty_eq1} the term $\eta \E[\sum_{t=1}^T h_t]$
and dividing both sides by $\eta$, and plugging $a = X^2$, we get that for all $\eta > 0$,
\begin{align*}
  &\E\left[ \sum_{t=1}^T h_t\right]
  \leq \E\left[\sum_{t=1}^T h_t \cdot (1 - \inn{\bm U}{(\bm e_{y_t} - \bm e_{\bar y_t}) \otimes \bm x_t}) \right. \\
  &\quad + \left. \sum_{t=1}^T h_t \cdot \frac{\eta}{2}\inn{\bm U}{(\bm e_{\bar{y}_t} - \bm e_{y_t}) \otimes \bm x_t}^2\right] \\
  &\quad + \frac{\eta X^2}{2} \| \bm U \|_F^2 + \frac{d\,k^2}{2\,\gamma\,\eta} \ln T \\
  &\leq \E\left[\sum_{t=1}^T \ell_0(\bm U, (\bm x_t, y_t)) + \left(\sum_{t=1}^T h_t + \frac12\right) \cdot \eta \|\bm U\|_F^2 X^2\right]  \\
  &\quad + \frac{d\,k^2}{2\,\gamma\,\eta} \ln T ~.
\end{align*}
where the first inequality uses Lemma~\ref{lem:adaptivenorms}, the second
inequality is from Cauchy-Schwarz that
$\inn{\bm U}{(\bm e_{y_t} - \bm e_{\bar y_t}) \otimes \bm x_t} \leq  \| \bm U \|_F \cdot \|(\bm e_{y_t} - \bm e_{\bar y_t}) \otimes \bm x_t \| \leq \|U\|_F \sqrt{2} X$ and that
$(1 - \inn{U}{(\bm e_{y_t} - \bm e_{\bar y_t}) \otimes \bm x_t}) \leq \ell(\bm U, (\bm x_t, y_t))$.

Taking $\eta = \frac{\frac{dk^2}{2\gamma} \ln T}{\|\bm U\|_F^2 (\E[\sum_t h_t] + \frac12) X^2}$,
we have
\begin{align*}
&\E\left[\sum_{t=1}^T h_t\right]
\leq L_0(\bm U)\\
&\quad+\sqrt{ \|\bm U\|_F^2 \left(\E\left[\sum_{t=1}^T h_t\right] + \frac 1 2\right)\frac{d\,k^2\,X^2}{2\gamma} \ln T}\\
&\leq L_0(\bm U) + \sqrt{ \|\bm U\|_F^2 \left( \E\left[\sum_{t=1}^T h_t\right] + \frac12\right) \frac{d\,k^2\,X^2}{\gamma} \ln T},
\end{align*}
where the last inequality is due to the elementary inequality $\sqrt{c+d}\leq \sqrt{c}+\sqrt{d}$, and the
setting of $a = X^2$. Solving the inequality and using the fact that $\E[M_T]\leq \E[\hat{M}_T]+\gamma T \leq \E[\sum_{t=1}^T h_t] + \gamma T$, we have
\resizebox{\linewidth}{!}{
  \begin{minipage}{\linewidth}
  \begin{align*}
    &\E[M_T] \leq L_0(\bm U) + \gamma T \\
    & + O\left( \frac{d\,k^2 \, \|\bm U\|_F^2 \, X^2 \ln T }{\gamma}+  \sqrt{ L_0(\bm U) \frac{d \, k^2 \, \|\bm U\|_F^2 \, X^2 \, \ln T }{\gamma} }\right).
  \end{align*}
  \end{minipage}
}

% \begin{align*}
% &\E[M_T] \leq L_0(\bm U) + \gamma T \\
% & + O\left( \frac{d\,k^2 \, \|\bm U\|_F^2 \, X^2 \ln T }{\gamma}+  \sqrt{ L_0(\bm U) \frac{d \, k^2 \, \|\bm U\|_F^2 \, X^2 \, \ln T }{\gamma} }\right)~.
% \end{align*}
The theorem follows from Lemma~\ref{lem:fallbacktuning} in Appendix~\ref{sec:appendix1},
taking $U = \|\bm U\|_F^2$, $H = d \, k^2 \, X^2 \, \ln T$, $L = L_0(\bm U)$.
\end{proof}

%% file: experiments.tex
\section{Empirical Results}
\label{sec:exp}

\begin{figure*}[h!]
  \begin{tabular}{cccc}
  \includegraphics[width=0.22\textwidth]{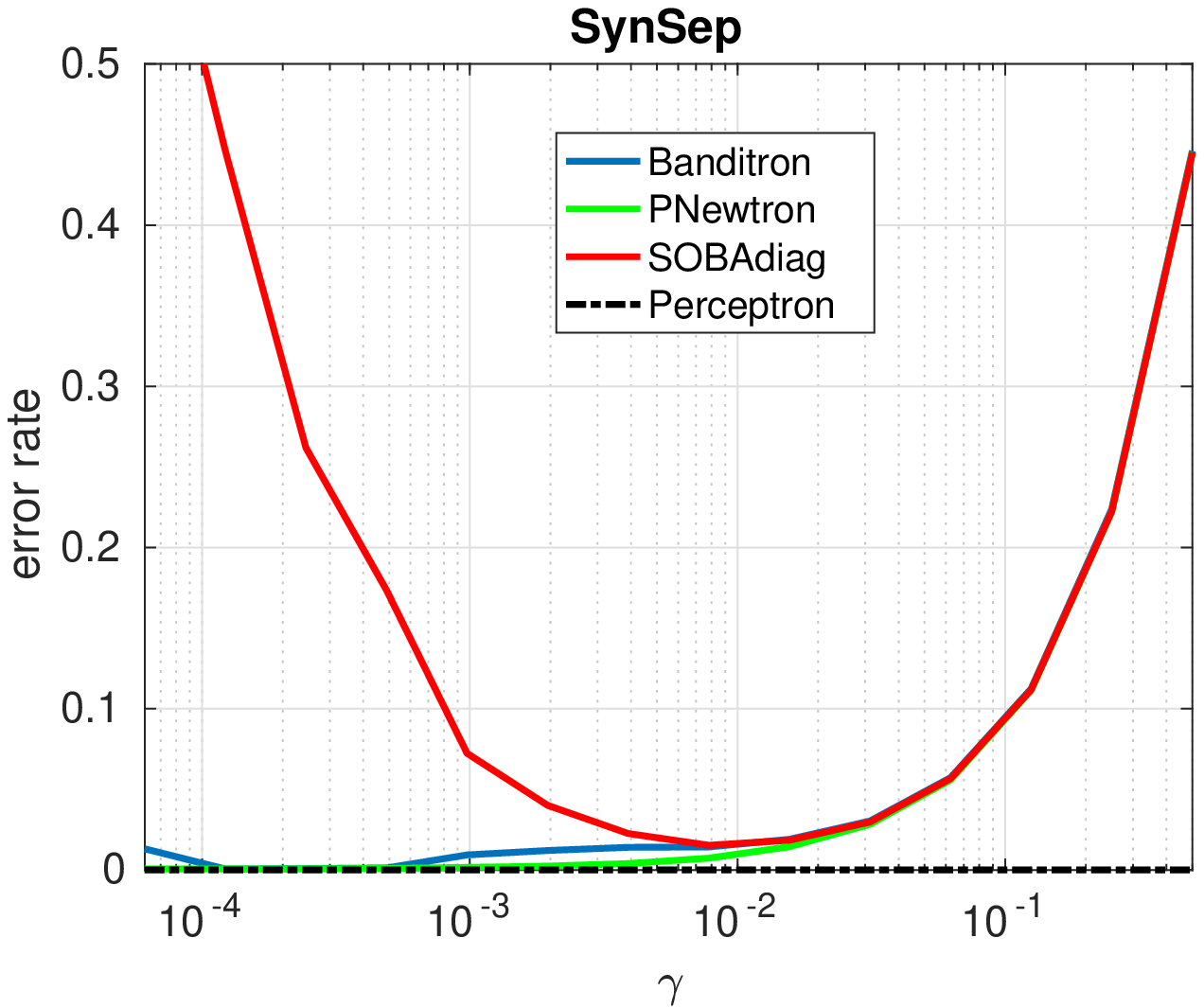} &
  \includegraphics[width=0.22\textwidth]{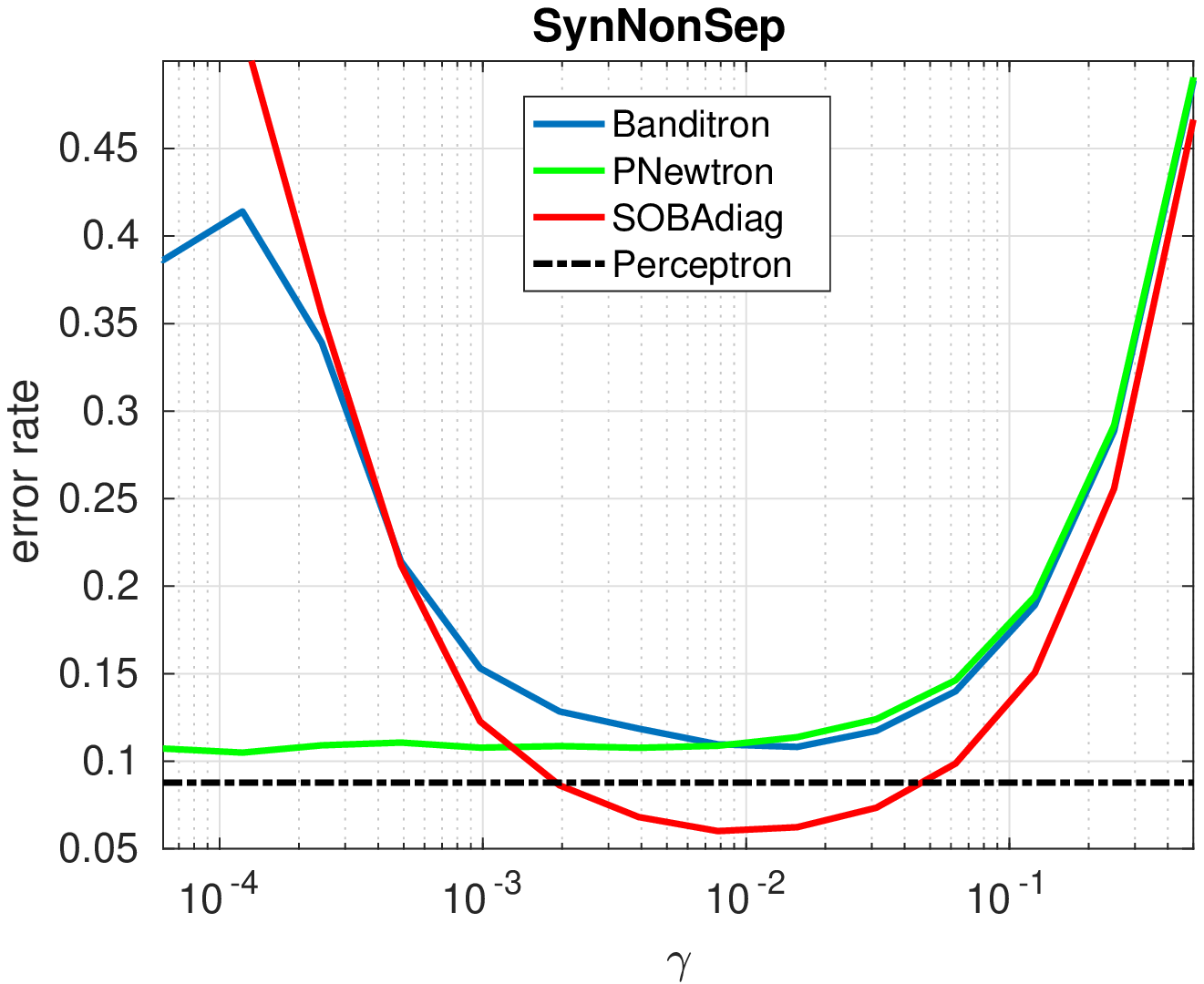} &
  \includegraphics[width=0.22\textwidth]{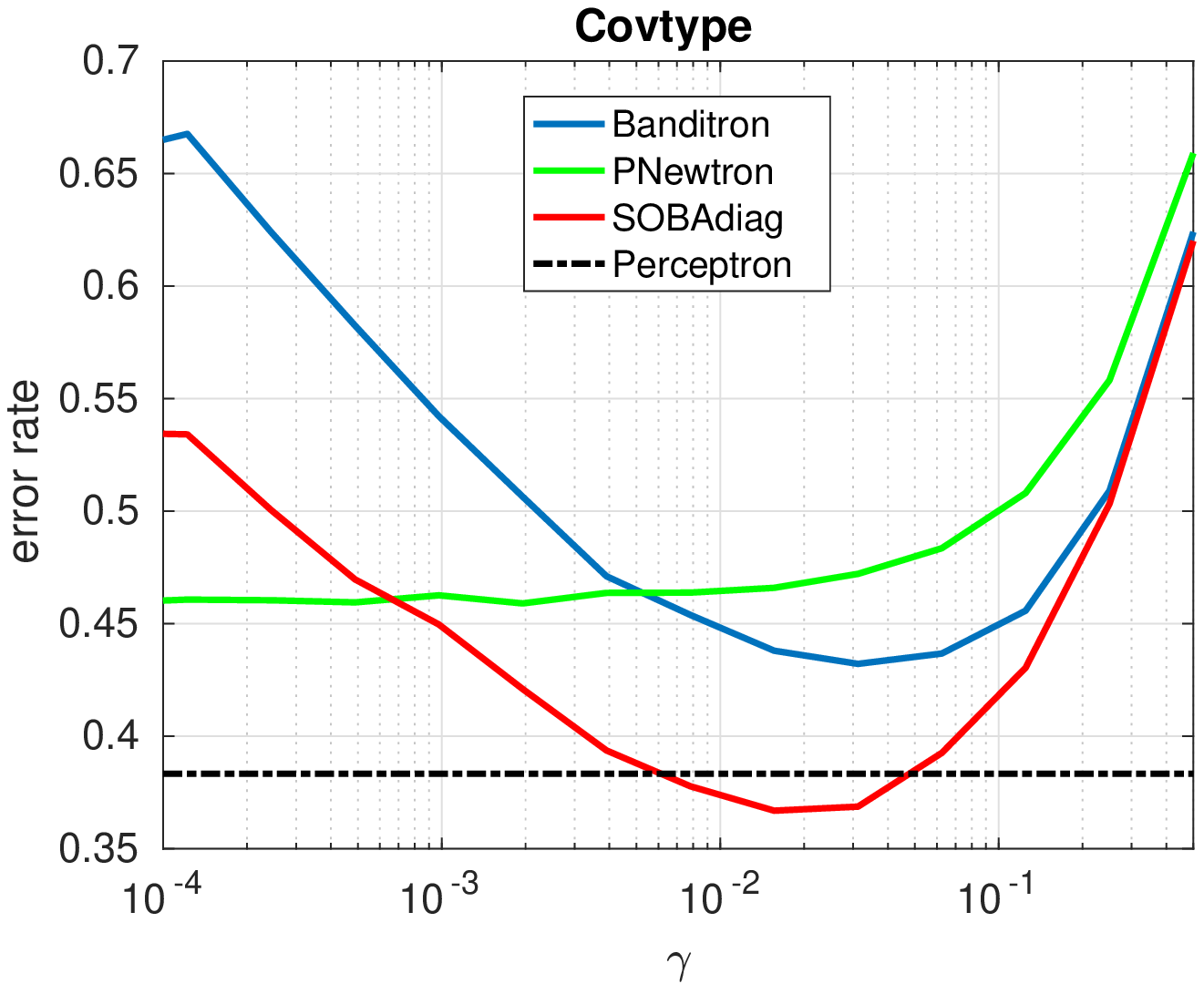} &
  \includegraphics[width=0.22\textwidth]{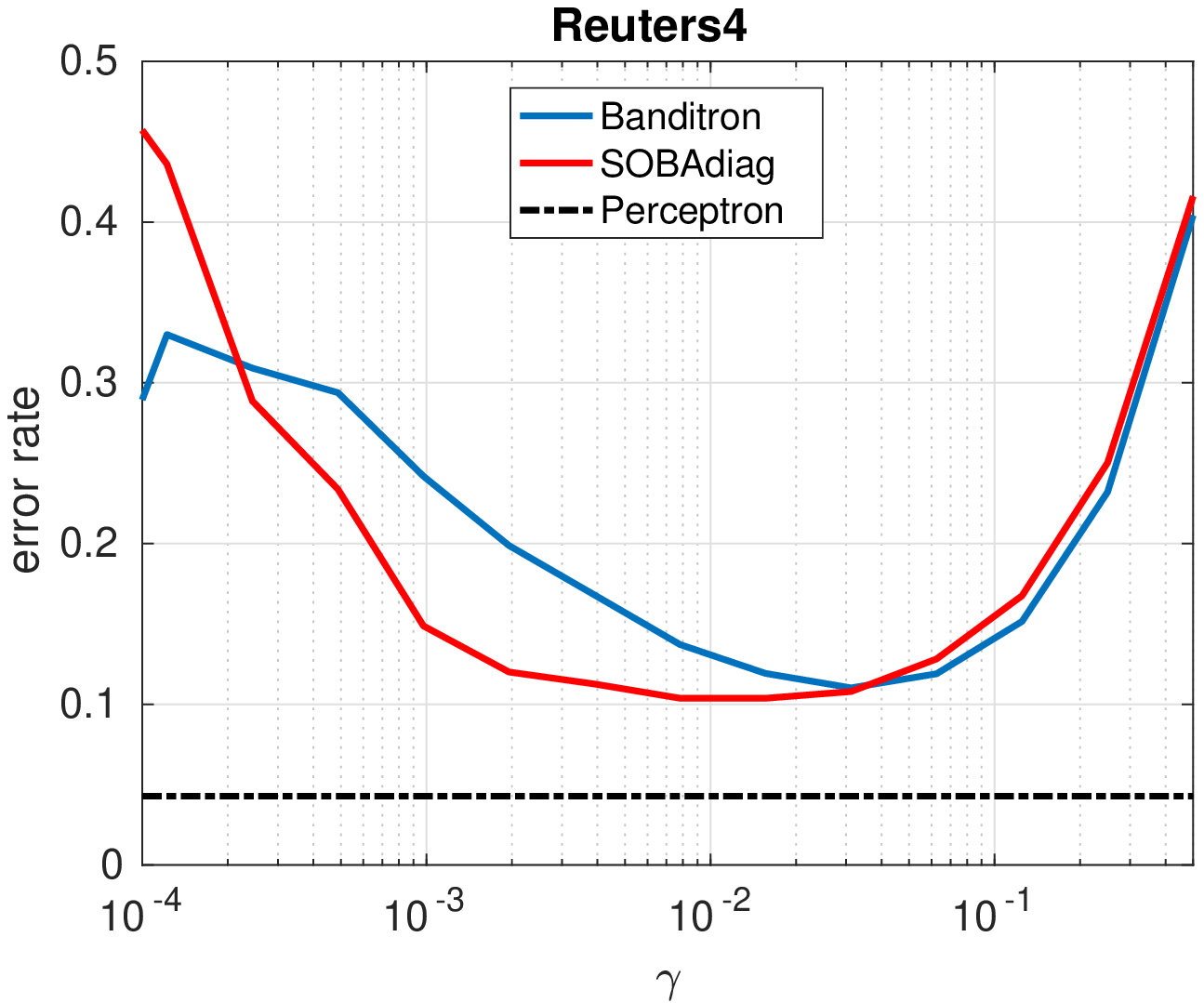} \\
  \includegraphics[width=0.22\textwidth]{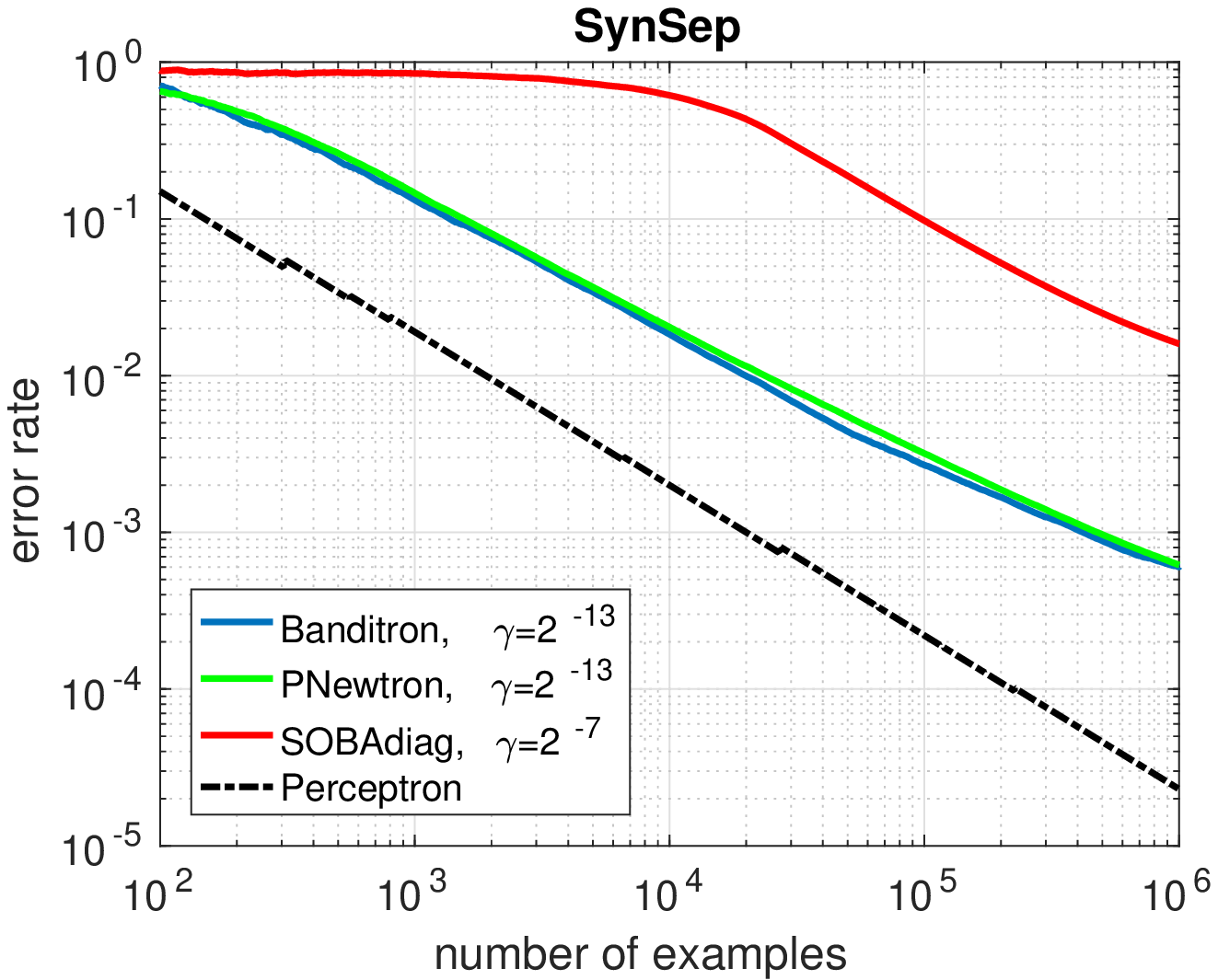} &
  \includegraphics[width=0.22\textwidth]{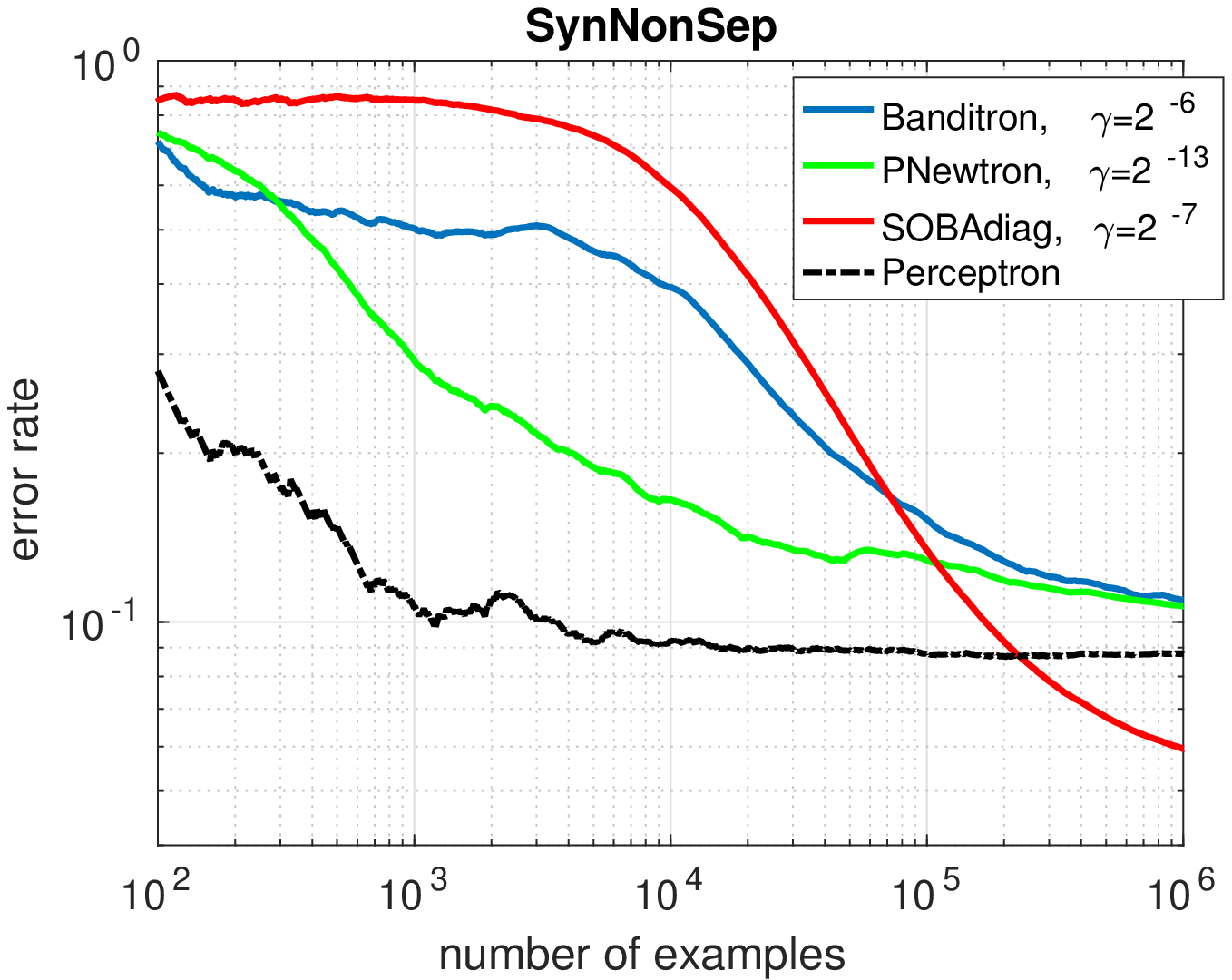} &
  \includegraphics[width=0.22\textwidth]{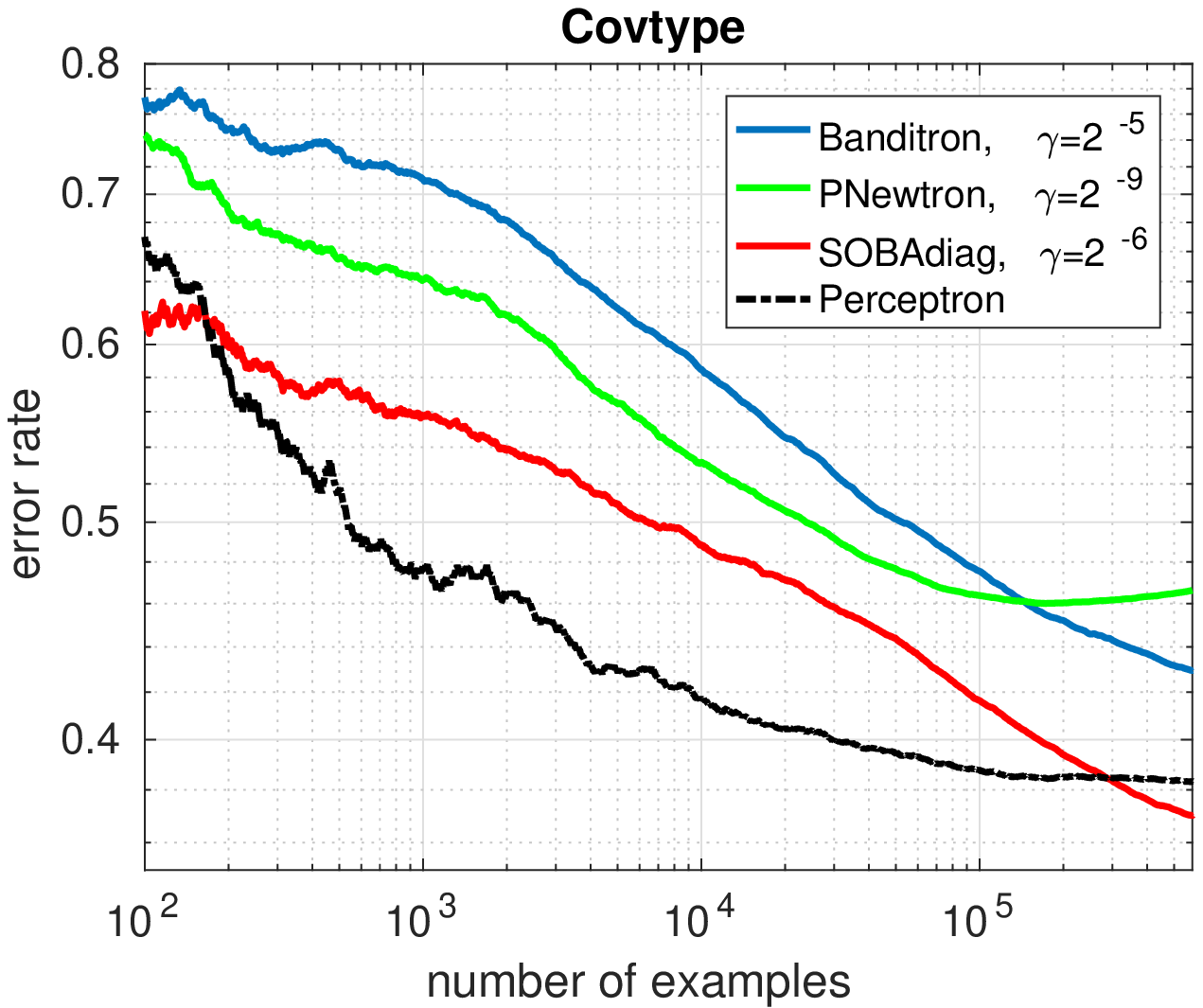} &
  \includegraphics[width=0.22\textwidth]{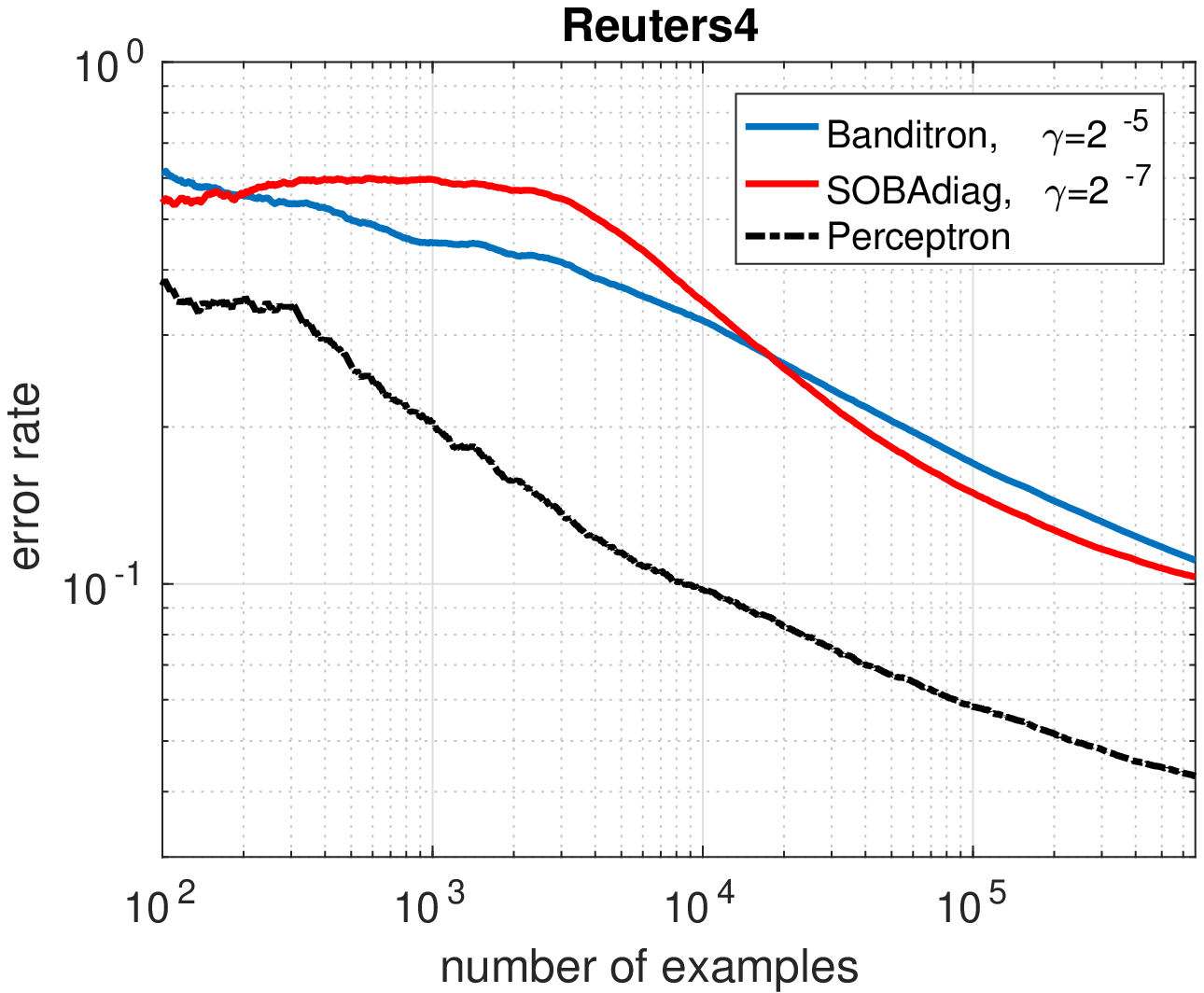}\\
  \end{tabular}
  \label{fig:exp1}
  \caption{Error rates vs. the value of the exploration rate $\gamma$ (top row) and vs. the number examples (bottom row). The x-axis is logarithmic in all the plots, while the y-axis is logarithmic in the plots in the second row. Figure best viewed in colors.}
\end{figure*}

We tested SOBA to empirically validate the theoretical findings.
We used three different datasets from \citet{Kakade-Shalev-Shwartz-Tewari-2008}: \texttt{SynSep, SynNonSep, Reuters4}. The first two are synthetic, with $10^6$ samples in $\R^{400}$ and 9 classes. \texttt{SynSep} is constructed to be linearly separable, while \texttt{SynNonSep} is the same dataset with 5\% random label noise. \texttt{Reuters4} is generated from the RCV1 dataset~\cite{Lewis04}, extracting the 665,265 examples that have exactly one label from the set \{\texttt{CCAT, ECAT, GCAT, MCAT}\}. It contains 47,236 features.
We also report the performance on \texttt{Covtype} from LibSVM repository.\footnote{\url{https://www.csie.ntu.edu.tw/~cjlin/libsvmtools/datasets/}}
We report averages over 10 different runs.

SOBA, as the Newtron algorithm, has a quadratic complexity in the dimension of the data, while the Banditron and the Perceptron algorithm are linear. Following the long tradition of similar algorithms~\citep{CrammerK09,DuchiHS11,HazanK11,CrammerG13}, to be able to run the algorithm on large datasets, we have implemented an approximated diagonal version of SOBA, named SOBAdiag. It keeps in memory just the diagonal of the matrix $A_t$.
Following \citet{HazanK11}, we have tested only algorithms designed to work in the fully adversarial setting. Hence, we tested the Banditron and the PNewtron, the diagonal version of the Newtron algorithm in \citet{HazanK11}. The multiclass Perceptron algorithm was used as a full-information baseline.

In the experiments, we only changed the exploration rate $\gamma$, leaving fixed all the other parameters the algorithms might have.
In particular, for the PNewtron we set $\alpha=10$, $\beta=0.01$, and $D=1$, as in \citet{HazanK11}. In SOBA, $a$ is fixed to 1 in all the experiments.
We explore the effect of the exploration rate $\gamma$ in the first row of Figure \ref{fig:exp1}.
We see that the PNewtron algorithm,\footnote{We were unable to make the PNewtron work on \texttt{Reuters4}. For any setting of $\gamma$ the error rate is never better than 57\%. The reason might be that the dataset RCV1 has 47,236 features, while the one reported in \citet{Kakade-Shalev-Shwartz-Tewari-2008, HazanK11} has 346,810, hence the optimal setting of the 3 other parameters of PNewtron might be different. For this reason we prefer not to report the performance of PNewtron on \texttt{Reuters4}.} thanks to the exploration based on the softmax prediction, can achieve very good performance for a wide range of $\gamma$. 

It is important to note that SOBAdiag has good performance on all four datasets for a value of $\gamma$ close to 1\%. For bigger values, the performance degrades because the best possible error rate is lower bounded by $\frac{k-1}{k}\gamma$ due to exploration. For smaller values of exploration, the performance degrades because the algorithm does not update enough. In fact, SOBA updates only when $\tilde{y}_t=y_t$, so when $\gamma$ is too small the algorithms does not explore enough and remains stuck around the initial solution. Also, SOBA requires an initial number of updates to accumulate enough negative terms in the $\sum_t n_t m_t$ in order to start updating also when $\hat y_t$ is correct but the margin is too small.

The optimal setting of $\gamma$ for each algorithm was then used to generate the plots in the second row of Figure~\ref{fig:exp1}, where we report the error rate over time. With the respective optimal setting of $\gamma$, we note that the performance of PNewtron does not seem better than the one of the Multiclass Perceptron algorithm, and on par or worse to the Banditron's one. On the other hand, SOBAdiag has the best performance among the bandits algorithms on 3 datasets out of 4.

The first dataset, \texttt{SynSep}, is separable and with their optimal setting of $\gamma$, all the algorithms converge with a rate of roughly $O(\frac{1}{T})$, as can be seen from the log-log plot, but the bandit algorithms will not converge to zero error rate, but to $\frac{k-1}{k}\gamma$. However, SOBA has an initial phase in which the error rate is high, due to the effect mentioned above.

On the second dataset, \texttt{SynNonSep}, SOBAdiag outperforms all the other algorithms (including the full-information Perceptron), achieving an error rate close to the noise level of 5\%. This is due to SOBA being a second-order algorithm, while the Perceptron is a first-order algorithm.
%With the optimal settings of $\gamma$, the final error rate of the Perceptron over \texttt{SynNonSep} is 8.8\%, while it is 10.7\%, 10.67\%, 5.94\%, for the Banditron, PNewtron, and SOBAdiag respectevely.
A similar situation is observed on \texttt{Covtype}.
On the last dataset, \texttt{Reuters4}, SOBAdiag achieves performance better than the Banditron.

%% file: discussion.tex
\section{Discussion and Future Work}
\label{sec:discussion}
In this paper, we study the
problem of online multiclass learning with bandit feedback.
We propose SOBA, an algorithm that achieves a regret of
$\tilde O(\frac1\eta \sqrt{T})$ with respect to $\eta$-loss of the competitor.
This answers a COLT open problem posed by~\cite{AbernethyR09}.
Its key ideas are to apply a novel adaptive regularizer in a second order online learning algorithm, coupled with updates only when the predictions are correct. SOBA is shown to have competitive performance compared
to its precedents in synthetic and real datasets, in some cases even better than the full-information Perceptron algorithm.
There are several open questions we wish to explore:
%\begin{enumerate}

%\item
1. Is it possible to design efficient algorithms with mistake bounds that depend on the loss of the competitor,
i.e. $\E[M_T] \leq L_\eta(\bm U) + \tilde O(\sqrt{ kd L_\eta(\bm U) } + kd)$? This
type of bound occurs naturally in the full information multiclass online learning setting,
%e.g. multiclass Perceptron mistake bound
(see e.g. Theorem~\ref{thm:perc_bound}),
or in multiarmed bandit setting, e.g.~\cite{Neu15}.
%We remark that it is unknown (?) whether there exist contextual bandit algorithms that can achieve relative loss bounds.

%\item
2. Are there efficient algorithms that have a finite mistake bound in
the separable case? ~\cite{Kakade-Shalev-Shwartz-Tewari-2008} provides an
algorithm that performs enumeration and plurality vote to achieve a finite mistake bound in the finite dimensional setting,
but unfortunately the algorithm is impractical.
Notice that it is easy to show that in SOBA $\hat{y}_t$ makes a logarithmic number of mistakes in the separable case, with a constant rate of exploration, yet it is not clear how to decrease the exploration over time in order to get a logarithmic number of mistakes for $\tilde{y}_t$.
%One direction may be use Ellipsoid method
%to efficiently a the version space in multiclass bandit setting, as in~\cite{YangJY09}
%\end{enumerate}
%in binary classification setting.

%% file: adaptive.tex
\section{Adaptive Tuning of the Exploration Rate}
\label{sec:adaptive}

In Theorem~\ref{thm:soba_bound} we have presented a tuning of $\gamma$ that guarantees a regret of the order of $\tilde O(\frac{1}{\eta}\sqrt{T})$.
However, this setting requires to upper bound the sum of the quadratic terms with a worst case bound.
In this section, we develop an adaptive strategy for the tuning of the exploration rate $\gamma$ that guarantees an optimal bound w.r.t. to the tightest sum of the quadratic terms.

First, we make rate dependent of the time, i.e. $\gamma_t$.
Our aim is to choose $\gamma_t$ in each time step in order to minimize the
excess mistake bound
$\E\left[ \sum_{t=1}^T\gamma_t + \frac{1}{\eta(2-\eta)} \sum_{t=1}^T \frac{k}{\gamma_t} \bm z_t^T \bm A_t^{-1} \bm z_t \right]$.
The main result is that, adaptively setting $\gamma_t$'s would result in a bound
within (roughly) a constant factor of that obtained by the best fixed
$\gamma$ in hindsight. We start with a technical lemma.

\begin{lemma}
\label{lem:selfconfident}
Let $c_1, \ldots, c_T \in [0,b]$ be a sequence of real numbers, $a>0$, and define
$\gamma_t = \min\left(\sqrt{\tfrac{b+\sum_{s=1}^{t-1}c_s}{t}},1\right)$. We have,
\[
\sum_{t=1}^T \left(\gamma_t + a \frac{c_t}{\gamma_t}\right) \leq (2+2a) \sqrt{T} \sqrt{b+\sum_{t=1}^{T}c_t} + a \sum_{t=1}^T c_t~.
\]
\end{lemma}
\begin{proof}
First, note that
\begin{align*}
   \sum_{t=1}^T \gamma_t \leq \sum_{t=1}^T \sqrt{\frac{b+\sum_{s=1}^{t-1}c_s}{t}} \leq \sqrt{b+\sum_{s=1}^T c_s} \sum_{t=1}^T\sqrt{\frac{1}{t}}
   \leq 2\sqrt{T} \sqrt{b+\sum_{s=1}^T c_s}~.
\end{align*}
Second, using the elementary chain of inequalities $\max(a,b)\leq a+b, \forall a,b\geq0$, we have that
\begin{align*}
\sum_{t=1}^T \frac{c_t}{\gamma_t}
   &= \sum_{t=1}^T \max\left( \frac{c_t\sqrt{t}}{\sqrt{b+\sum_{s=1}^{t-1}c_s}}, c_t\right)\\
   &\leq \sum_{t=1}^T \sqrt{T} \frac {c_t}  {\sqrt{b+\sum_{s=1}^{t-1}c_s}} +  \sum_{t=1}^T c_t \\
   &\leq \sqrt{T} \sum_{t=1}^T \frac {c_t}  {\sqrt{\sum_{s=1}^{t}c_s}} + \sum_{t=1}^T c_t\\
   &\leq 2\sqrt{T} \sqrt{b+\sum_{s=1}^T c_s} + \sum_{t=1}^T c_t,
\end{align*}
where the last inequality uses Lemma 3.5 of~\cite{AuerCG02}.
Combining the two inequalities, we get the desired result.
\end{proof}

Built upon the lemma above, we show that, tailored to our setting, the adaptive
tuning would result in a bound within a constant factor of that achieved by the best fixed $\gamma$ in hindsight.
\begin{theorem}
Running SOBA with the adaptive setting of $\gamma_t = \min\left(\sqrt{\frac{k(1+ \sum_{s=1}^{t-1} \bm z_s^T \bm A_s^{-1} \bm z_s)}{t}},1\right)$ and $a = X^2$, we have that
\begin{align*}
  &\E[M]
  %&\leq L_\eta(U)+ a \| U \|_F^2 + \sum_{t=1}^T \frac{k}{\gamma_t} \bm z_t^T \bm A_t^{-1} \bm z_t + \sum_{t=1}^T \gamma_t\\
  \leq L_\eta(U) +O\left(X^2 \| U \|_F^2+\frac{1}{\eta} (\sqrt{d k^2 T \ln T} + dk^2 \ln T) \right)~.
  %&=& O(\ln T) \min_{\gamma \in [0,1]} \ T\gamma + \frac{k^2}{\gamma}\left(\ln 2 + \sum_{i=1}^d \ln\left(1 + \frac{\lambda_i}{a \gamma}\right)\right)
\end{align*}
\end{theorem}
\begin{proof}[Proof Sketch]
Following the same proof as Theorem~\ref{thm:yhat}, we get that
\[
\E\left[\hat{M}_T \right]
\leq L_\eta(\bm U)+\frac{a\eta \| \bm U \|_F^2}{2-\eta}
 + \frac{1}{\eta(2-\eta)} \E[\sum_{t=1}^T\frac{k}{\gamma_t}\bm z_t^T \bm A_t^{-1} \bm z_t]
\]
Meanwhile by triangle inequality,
 \begin{equation*}
   \E[M_T] \leq \E[\hat{M}_T] + \E \left[\sum_{t=1}^T \one[\tilde y_t \neq \hat y_t] \right] \leq \E[\hat{M}_T] + \E \left[\sum_{t=1}^T \gamma_t \right]~.
 \end{equation*}
Combining the two inequalities above, we get
\[ \E\left[M_T \right]
    \leq L_\eta(\bm U)+\frac{a\eta \| \bm U \|_F^2}{2-\eta}
 + \E\left[\frac{1}{\eta(2-\eta)} \sum_{t=1}^T\frac{k \, \bm z_t^T \bm A_t^{-1} \bm z_t}{\gamma_t} + \sum_{t=1}^T \gamma_t\right].\]
We take a closer look at the last term. Lemma~\ref{lem:selfconfident} with $c_t = k \bm z_t^T \bm A_t^{-1} \bm z_t \in [0,k]$, $b = k$, $a = \frac{1}{\eta(2-\eta)}$,
implies that
\begin{align*}
\sum_{t=1}^T &\gamma_t + \sum_{t=1}^T \frac{k}{\eta (2-\eta)\gamma_t} \bm z_t^T \bm A_t^{-1} \bm z_t\\
   &\leq \left(2+\frac{2}{\eta(2-\eta)}\right) \sqrt{T}\sqrt{k (1 + \sum_{t=1}^{T} \bm z_t^T \bm A_t^{-1} \bm z_t)} + \frac{1}{\eta(2-\eta)} \, k (1 + \sum_{t=1}^{T} \bm z_t^T \bm A_t^{-1} \bm z_t)~.
\end{align*}
Taking the expecation of both sides and using Lemma~\ref{lem:adaptivenorms},
we get that the last term on the right hand side is at most $\frac{12}{\eta} (\sqrt{d k^2 T \ln T} + dk^2 \ln T)$.
This completes the proof.
% Reasoning as in the proof of Lemma~\ref{lem:vaw}, we have that
% \[
% \sum_{t=1}^{T} \tr(\bm A_t^{-1} \bm z_t \bm z_t^T) \leq d k \ln\left(1 + \frac{2X^2 \sum_{i=1}^T \frac{1}{\gamma_t}I[y_t \neq \hat{y}_t]}{a\,d}\right)~.
% \]
\end{proof}

%% file: proof_perc_q.tex
\section{Deferred Proofs}
\label{sec:appendix1}

\begin{proof}[Proof of Theorem~\ref{thm:perc_bound}] \label{proof:thm:2}
Let $p\geq2$ such that $\tfrac{1}{p}+\tfrac{1}{q}=1$. Denote by $b_t$ the
indicator variable that multiclass Perceptron makes an update, i.e. makes a mistake.
We have:

\begin{align*}
\langle &\bm W_{T+1}, \bm U \rangle\\
&\leq \norm{\bm W_{T+1}}_F \norm{\bm U}_F \\
&= \norm{\bm U}_F \sqrt{\norm{\bm W_{T}}^2+2 b_t \langle \bm W_T, (\bm e_{y_T} - \bm e_{\hat{y}_T}) \otimes \bm x_T \rangle + 2b_t^2 \norm{\bm x_T}_2^2} \\
&\leq \norm{\bm U}_F \sqrt{\norm{\bm W_{T}}_F^2+2b_t^2 \norm{\bm x_T}_2^2} \\
&\leq \cdots \\
&\leq \norm{\bm U}_F \sqrt{2 \sum_{t=1}^{T} b_t^2 \norm{\bm x_t}_2^2} \\
&\leq \norm{\bm U}_F X \sqrt{2}\sqrt{\sum_{t=1}^{T} b_t^2} \\
&= \norm{\bm U}_F X \sqrt{2} \sqrt{\sum_{t=1}^{T} b_t}
\end{align*}
Also, we have, that
\begin{align*}
\langle \bm W_{T+1}, \bm U \rangle
&= \sum_{t=1}^T b_t \langle \bm U, (\bm e_{y_t}-\bm e_{\hat{y}_t}) \otimes \bm x_t \rangle \\
&= \sum_{t=1}^T b_t [ 1-(1- \langle \bm U, (\bm e_{y_t}-\bm e_{\hat{y}_t}) \otimes \bm x_t \rangle)] \\
&\geq \sum_{t=1}^T b_t [ 1-|1- \langle \bm U, (\bm e_{y_t}-\bm e_{\hat{y}_t}) \otimes \bm x_t \rangle|_+] \\
&\geq \sum_{t=1}^T b_t - \sum_{t=1}^T b_t \ell(\bm U, (\bm x_t, y_t)) \\
&\geq \sum_{t=1}^T b_t - (\sum_{t=1}^T b_t^p)^\frac{1}{p} (\sum_{t=1}^T \ell(\bm U, (\bm x_t, y_t))^q)^\frac{1}{q} \\
&= \sum_{t=1}^T b_t - (\sum_{t=1}^T b_t)^\frac{1}{p} (\sum_{t=1}^T \ell(\bm U, (\bm x_t, y_t))^q)^\frac{1}{q} \; .
\end{align*}
Putting all together we have
\begin{align*}
\norm{\bm U}_F X \sqrt{2} \sqrt{\sum_{t=1}^{T} b_t} \geq \sum_{t=1}^T b_t - \left(\sum_{t=1}^T b_t\right)^\frac{1}{p} L_{\text{MH},q}(\bm U)^\frac{1}{q}~.
\end{align*}
Noting that $\sum_{t=1}^{T} b_t$ is equal to number of mistake $M_T$, we get the stated bound.
\end{proof}

\begin{lemma}
Suppose we are given positive real numbers $L, T, H, U$ and function
 $F(\gamma) = \min(T, L + \gamma T + \frac{UH}{\gamma} + \sqrt{\frac{UHL}{\gamma}})$,
where $\gamma \in [0,1]$. Then:
\begin{enumerate}
\item If $L \leq (U+1)\sqrt{HT}$, then taking $\gamma^* = \min(\sqrt{\frac H T}, 1)$ gives that $F(\gamma^*) \leq L + 3(U+1)\sqrt{HT}$.
\item If $L > (U+1)\sqrt{HT}$, then taking $\gamma^* = \min((\frac {HL} {T^2})^{\frac 1 3}, 1)$ gives that $F(\gamma^*) \leq L + 2(\sqrt{U} + 1) (HLT)^{\frac13}$.
\end{enumerate}
\label{lem:fallbacktuning}
\end{lemma}
\begin{proof}
We prove the two cases separately.
\begin{enumerate}
\item If $T \leq H$, then $\gamma^* = 1$, $F(\gamma^*) \leq T \leq  L + 3(U+1)\sqrt{HT}$.

Otherwise, $T > H$. In this case, $\gamma^* = \sqrt{\frac H T}$. We have that
\begin{eqnarray*}
&&F(\gamma^*) \\
&=& L + \gamma^* T + \frac{UH}{\gamma^*} + \sqrt{\frac{UHL} {\gamma^*}} \\
&=& L + \sqrt{HT} + U\sqrt{HT} + \sqrt{UL \sqrt{HT}} \\
&\leq& L + (U + 1) \sqrt{HT} + L + U \sqrt{HT} \\
&\leq& L + 3(U + 1) \sqrt{HT}.
\end{eqnarray*}
where the first inequality is from that arithmetic mean-geometric mean inequality, the second inequality is
by the assumption on $L$.

\item If $HL > T^2$, then $\gamma^* = 1$, $F(\gamma^*) \leq T \leq (HLT)^{\frac 1 3}$.

Otherwise, $HL \leq T^2$. In this case, $\gamma^* = (\frac {HL} {T^2})^{\frac 1 3}$. We have that
\begin{eqnarray*}
F(\gamma^*)
&=& L + \gamma^* T + \frac{UH}{\gamma^*} + \sqrt{\frac{UHL} \gamma^*} \\
&=& L + (HLT)^{\frac 1 3} + UH^{\frac 2 3} T^{\frac 2 3} L^{-\frac 1 3} + \sqrt{U} (HLT)^{\frac 1 3} \\
&\leq& L + (\sqrt{U} + U^{\frac 1 3} + 1) (HLT)^{\frac 1 3} \\
&\leq& L + 2(\sqrt{U} + 1) (HLT)^{\frac 1 3}.
\end{eqnarray*}
where the first inequality is from algebra and the condition on $L$, implying
$UH^{\frac 2 3} T^{\frac 2 3} L^{-\frac 1 3} \leq (HLT)^{\frac 1 3} U (\frac{HT}{L^2})^{\frac13} \leq U^{\frac13} (HLT)^{\frac13}$,
the second inequality is from that $U^{\frac 1 3} \leq \sqrt{U} + 1$.
\end{enumerate}
\end{proof}

\section{Per-Step Analysis of Online Least Squares}
For completeness, we present a technical lemma in online least squares, which has appeared in~\citep[e.g.,][]{Orabona-Cesa-Bianchi-Gentile-2012}.
\begin{lemma}
Suppose $\bm z_t$'s are vectors, and $\alpha_t$'s are scalars. For all $t \geq 1$, define
$\bm A_t = \sum_{s=1}^t \bm z_s \bm z_s^T$, $\bm w_t = -\bm A_{t-1}^{-1} \sum_{s=1}^{t-1} \alpha_s \bm z_s$. Then
for any vector $\bm u$, we have:
\begin{equation*}
  \frac12(\inn{\bm w_t}{\bm z_t} + \alpha_t)^2 (1 - \bm z_t^T \bm A_t^{-1} \bm z_t)
  - \frac12(\inn{\bm u}{\bm z_t} + \alpha_t)^2 \leq
  \frac 1 2 \| \bm u - \bm w_t \|_{\bm A_{t-1}}^2 - \frac 1 2 \| \bm u - \bm w_{t+1} \|_{\bm A_t}^2~.
\end{equation*}
\label{lem:perstep}
\end{lemma}

\begin{proof}
Observe that $\bm w_t$'s have the following recurrence:
\begin{eqnarray*}
  \bm w_{t+1} = \bm A_t^{-1} (\bm A_{t-1} \bm w_t - \alpha_t \bm z_t)
\end{eqnarray*}
Since $\bm A_t = \bm A_{t-1} + \bm z_t \bm z_t^T$, we have
\[  \bm A_t \bm w_{t+1} = \bm A_t \bm w_t - (\bm w_t^T \bm z_t + \alpha_t)\bm z_t \]
Now, by standard online mirror descent analysis \citep[See e.g.][proof of Theorem 11.1]{Cesa-Bianchi-Lugosi-2006}, we have
\begin{eqnarray*}
  \inn{\bm w_t - \bm u}{(\bm w_t^T \bm z_t + \alpha_t) \bm z_t} &\leq& \frac 1 2 \| \bm u - \bm w_t \|_{\bm A_t}^2 - \frac 1 2 \| \bm u - \bm w_{t+1} \|_{\bm A_t}^2 + \frac 1 2 (\bm w_t^T \bm z_t + \alpha_t)^2 \bm z_t^T \bm A_t^{-1} \bm z_t \\
  &\leq& \frac 1 2 \| \bm u - \bm w_t \|_{\bm A_{t-1}}^2 - \frac 1 2 \| \bm u - \bm w_{t+1} \|_{\bm A_t}^2 + \frac 1 2 (\bm w_t^T \bm z_t + \alpha_t)^2 \bm z_t^T \bm A_t^{-1} \bm z_t + \frac 1 2 (\bm u^T \bm z_t - \bm w_t^T \bm z_t)^2
\end{eqnarray*}
Now, moving the last term on the RHS to the LHS, we get
\[ (\bm w_t^T \bm z_t- \bm u^T \bm z_t) \cdot \frac 1 2 (\bm w_t^T \bm z_t + \bm u^T \bm z_t + 2\alpha_t) \leq \frac 1 2 \| \bm u - \bm w_t \|_{\bm A_{t-1}}^2 - \frac 1 2 \| \bm u - \bm w_{t+1} \|_{\bm A_t}^2 + \frac 1 2 (\bm w_t^T \bm z_t + \alpha_t)^2 \bm z_t^T \bm A_t^{-1} \bm z_t\]
i.e.
\begin{equation*}
  \frac12(\inn{\bm w_t}{\bm z_t} + \alpha_t)^2
  - \frac12(\inn{\bm u}{\bm z_t} + \alpha_t)^2 \leq
  \frac 1 2 \| \bm u - \bm w_t \|_{\bm A_{t-1}}^2 - \frac 1 2 \| \bm u - \bm w_{t+1} \|_{\bm A_t}^2+ \frac 1 2 (\bm w_t^T \bm z_t + \alpha_t)^2 \bm z_t^T \bm A_t^{-1} \bm z_t ~.
\end{equation*}
Now moving the last term on the RHS to the LHS, the lemma follows.
\end{proof}

\section{Additional Discussions of Newtron~\citep{HazanK11}}
\label{sec:newtron}
%Although the motivation of the Newtron algorithm~\citep{HazanK11} is to give a
%regret bound in terms of the log loss parameterized by $\alpha$,
We show in this section that the Newtron algorithm~\citep{HazanK11}
 can also be interpreted as
one that achieves a $\tilde{O}(\sqrt{T})$ regret, in the sense that
for a certain set of convex loss functions that upper bounds the 0-1 loss (defined below),
the difference between
the cumulative 0-1 loss of the algorithm and the cumumlative convex loss
of the best linear predictor is at most $\tilde{O}(\sqrt{T})$.
%This result shows that, besides the original motivation in performing log
%loss regret minimization, Newtron's 0-1 loss
\begin{theorem}[0-1 loss upper bound of Newtron~\citep{HazanK11}]
Define the $\alpha$-logistic loss as $\ell_{\logistic, \alpha}(\bm W, (\bm x,y)) :=
\log_2 (1 + \sum_{j \neq y} \exp( \alpha ((\bm W \bm x)_j - (\bm W \bm x)_y)) )$. Suppose we are given a sequence
of examples $(\bm x_1,y_1) \ldots, (\bm x_n, y_n)$ such that for all $t$, $\| \bm x_t \| \leq X$. Then, with appropriate tuning of its parameters,
Newtron has the following regret bound for all $\bm U \in \R^{k \times d}$ such that $\| \bm U \|_F \leq D$:
\[ \E[M_T] - \sum_{t=1}^T \E[\ell_{\logistic, \alpha}(\bm U, (\bm x_t,y_t))] \leq \min\left\{c \exp(4\alpha XD) \ln T, 6 c XDT^{2/3}\right\}, \]
where $c = O(k^3 n)$ is a constant independent of $\alpha$.
\end{theorem}
%\ell_{0-1}(W_t, (x_t,y_t))
\begin{proof}
Using Corollary 5 of~\citep{HazanK11}, and observe that the $\alpha$-log loss defined therein is equal to $\frac{\ln 2 }{\alpha} \cdot \ell_{\logistic, \alpha}(\cdot)$, we have that
\[ \sum_{t=1}^T \E[\ell_{\logistic, \alpha}(\bm W_t, (\bm x_t,y_t))] - \sum_{t=1}^T \ell_{\logistic, \alpha}(\bm U, (\bm x_t,y_t)) \leq \min\{c \exp(4\alpha XD) \ln T, 6 c XDT^{2/3}\}. \]
for some constant $c = O(k^3 n)$.
The theorem follows from the fact that $\ell_{\logistic, \alpha}$ is an upper bound of the 0-1 loss.
\end{proof}

Specifically, if $\alpha \leq \frac{\ln T}{8XD}$, then with appropriate tuning of its parameters, Newtron has a $\tilde{O}(\sqrt{T})$ regret bound against the $\alpha$-logistic loss of the best linear classifier. However, we show in the lemma below that the $\alpha$-logistic loss for this range of $\alpha$ has the undesirable property that the loss on every example is at least $\tilde{\Omega}(T^{- \frac 1 4})$. In sharp contrast,  both the multiclass hinge loss (used by Banditron) and the $\eta$-loss (used by SOBA) has the property that if the data is separable by a margin of 1, then the loss is zero.
For instance, in the realizable setting, the $\alpha$-logistic loss of the best linear classifier is at least $\Omega(T^{\frac 3 4})$, implying that the 0-1 loss of Newtron can only be (loosely) bounded by $\Omega(T^{\frac 3 4})$. In this case, the mistake bound is worse than that given by Banditron or SOBA, which are both $\tilde O(\sqrt{T})$. %We leave the design of a better loss function to future work

\begin{lemma}
Suppose the loss parameter $\alpha$ is at most $\frac{\ln T}{8XD}$.
If we are given a linear classifier $\bm U \in \R^{k \times d}$ such that $\| \bm U \|_F \leq D$ and an example $(\bm x, y)$ such that $\| \bm x \| \leq X$, then the $\alpha$-logistic loss of $\bm U$ on $(\bm x,y)$, $\ell_{\logistic, \alpha}(\bm U, (\bm x,y))$, is $\Omega(T^{-\frac 1 4})$.
\end{lemma}
\begin{proof}
  %Recall that $\ell_{\logistic, \alpha}(W, (x,y))$.
  Observe that for $\bm U$ and $(\bm x,y)$, we have that for all $j \neq y$,
  \[ | \alpha ((\bm U \bm x)_j - (\bm U \bm x)_y) | \leq \frac {2 X D \ln T}{8 X D} = \frac {\ln T} 4 \]
  This implies that
  $\ell_{\logistic, \alpha}(\bm U, (\bm x,y))  = \log_2(1 + \sum_{j \neq y} \exp( \alpha ((\bm U \bm x)_j - (\bm U \bm x)_y)) ) \geq \log_2 (1 + (k-1) T^{-\frac 1 4}) = \Omega(T^{-\frac 1 4})$.
\end{proof}

\section{Connections to Online Exp-concave Optimization}
\label{sec:expconcave}

In this section, we present a (non-adaptive) variant of SOBA, namely Algorithm~\ref{alg:soba-mod}. Recall that in the original SOBA algorithm, we implicitly reduce the online classification problem to online least squares regression (Lemma~\ref{lem:vaw}), a problem well-studied in the literature~\citep{vovk2001competitive, azoury2001relative}. In contrast, Algorithm~\ref{alg:soba-mod} uses a black-box reduction to online exp-concave optimization, a generalization of the online least squares problem~\citep{Hazan-Agarwal-Kale-2007}.
Compared to SOBA, Algorithm~\ref{alg:soba-mod} has the advantage that it is conceptually much simpler, i.e.
it does not need to know the details of the underlying online optimization process.
However, it has two crucial drawbacks:
\begin{enumerate}
\item The adaptivtity of the algorithm is compromised. Existing exp-concave optimization oracles (defined below) need to know a bound on the competitor norm in advance~\citep[See e.g.][Lemma 3]{Hazan-Agarwal-Kale-2007}; in contrast, SOBA does not require such knowledge.
Moreover, the regret bound of Algorithm~\ref{alg:soba-mod} only holds for one $\eta$ chosen apriori (see Theorem~\ref{thm:soba-mod-reg} below); in contrast, the regret bound of SOBA holds for a range of $\eta$ simulateously, and the algorithm does not require the knowledge of $\eta$.
\item It is unclear how to incorporate the idea of ``passive-aggressive'' updates into the algorithm, which can substantially affect the algorithm's empirical performance.
\end{enumerate}
The results in this section is inspired by thought-provoking conversations with Satyen Kale~\citep{K17}.

\begin{algorithm}
\caption{Bandit Multiclass Classification via Reduction to Online Exp-concave Optimization}
\label{alg:soba-mod}
\begin{algorithmic}[1]
\REQUIRE Exploration parameter $\gamma \in [0,1]$, loss parameter $\eta$, online exp-concave optimzation oracle $\calO$.
\FOR{$t=1,2,\ldots,T$}
\STATE Receive instance $\bm x_t \in \R^d$
\STATE Receive $\bm W_t$ from optimization oracle $\calO$
\STATE $\hat{y}_t = \arg\max_{i \in [k]} (\bm W_t \bm x_t)_i$
\STATE Define $\bm p_t = (1-\gamma) \bm e_{\hat{y}_t} + \frac \gamma k \one_k$
\STATE Randomly sample $\tilde{y}_t$ according to $\bm p_t$
\STATE Receive bandit feedback $\one[\tilde{y}_t \not= y_t]$
\IF{$\tilde{y}_t = y_t$}
  \STATE Send loss function $\tilde{\ell}_t(\bm W) := \frac{1}{p_{t,y_t}} \cdot \one[\hat{y}_t \neq y_t] \cdot l_\eta(\bm W, (\bm x_t, y_t)) $ to optimization oracle $\calO$, where
  \[ l_\eta(\bm W, (\bm x_t, y_t)) := \left(1 - \frac{2}{2-\eta}((\bm W \bm x_t)_{y_t} - \max_{y \neq y_t} (\bm W \bm x_t)_t) + \frac{\eta}{2-\eta}((\bm W \bm x)_{y_t} - \max_{y \neq y_t} (\bm W \bm x)_t)^2 \right) \]
\ELSE
  \STATE Send loss function $\tilde{\ell}_t(\bm W) := 0$ to optimization oracle $\calO$
\ENDIF
\ENDFOR
\end{algorithmic}
\end{algorithm}

Specifically, Algorithm~\ref{alg:soba-mod} assumes access to an online exp-concave optimization oracle $\calO$, such that at each round $t$, it outputs a vector $\bm W_t \in \R^{k d}$, then receives a new loss function $\tilde{\ell}_t(\bm W)$ and updates its internal state. We require that $\calO$ achieves a low regret under certain conditions on the loss sequences. Formally:
\begin{assumption}[Efficient Exp-concave Optimization Oracle]
If all the $\tilde{\ell}_t$'s are $\beta$-exp-concave, and the subgradients of $\tilde{\ell}_t$'s are all $\ell_2$ bounded by $G$, then the $\bm W_t$'s output by $\calO$ satisfies that: for all $\bm U \in \R^{k \times d}$ such that $\| \bm U \|_F \leq D$,
\begin{equation}
  \sum_{t=1}^T \tilde{\ell}_t(\bm W_t) - \sum_{t=1}^T \tilde{\ell}_t(\bm U) \leq O\left( (GD + \frac 1 \beta) \cdot d \ln T \right).
  \label{eqn:expconcave}
\end{equation}
Moreover, the implementation of $\calO$ is computationally efficient.
\label{assumption:expconcave}
\end{assumption}
As we will see, the losses sent to $\calO$ in Algorithm~\ref{alg:soba-mod} are indeed exp-concave  and all have bounded subgradients (see Claim~\ref{claim:parameters}). The requirement of $\calO$ can be fulfilled by many algorithms, for example, the Online Newton Step algorithm and the Follow the Approximate Leader algorithm of~\cite{Hazan-Agarwal-Kale-2007}.

We show that given Assumption~\ref{assumption:expconcave} above, Algorithm~\ref{alg:soba-mod} is guaranteed to have a $\tilde{O}(\sqrt{T})$ regret bound. %against all linear predictors of bounded norm.
\begin{theorem}
Given an optimization oracle $\calO$ satisfiying Assumption~\ref{assumption:expconcave}, and suppose all the examples $\bm x_t$ have $\ell_2$ norm at most $X$. In addition, suppose positive constants $D$ and $\eta$ satisfies that
$\eta \leq \frac{1}{\max(2XD, 1)}$.
Then, the $\bm W_t$'s output by Algorithm~\ref{alg:soba-mod} is such that
for all $U$ such that $\| \bm U \|_F \leq D$,
\[ \sum_{t=1}^T \E[\ell_\eta(\bm W_t, (\bm x_t, y_t))] - \sum_{t=1}^T \ell_\eta(\bm U, (\bm x_t, y_t))
\leq O \left(  \frac{k^2 d}{\gamma \eta} \ln T \right). \]
Furthermore, taking $\gamma = \sqrt{\frac {k^2 d \ln T} {\eta T} }$, the $\tilde{y}_t$'s output by Algorithm~\ref{alg:soba-mod} is such that
\[ \E[M_T] - \sum_{t=1}^T \ell_\eta(\bm  U, (\bm x_t, y_t))
\leq O \left(  \sqrt{\frac{k^2 d \cdot T \ln T}{ \eta}} \right). \]
\label{thm:soba-mod-reg}
Recall that $M_T = \sum_{t=1}^T \one[\tilde{y}_t \neq y_t]$ is the cumulative 0-1 loss of the algorithm.
\end{theorem}
%\[ \E[\sum_{t=1}^T \ell_\eta(W_t, (x_t, y_t))] - \sum_{t=1}^T \ell_\eta(W^*, (x_t, y_t))
%\leq O \left(  \frac{k^2 d}{\gamma \eta} \ln T + \gamma T \right) \]
\begin{proof}
From the description of Algorithm~\ref{alg:soba-mod}, it can be seen that
\[
\tilde{\ell}_t(W) = \frac{\one[\tilde{y}_t = y_t]}{p_{t,y_t}} \cdot \one[\hat{y}_t \neq y_t] \cdot l_\eta(\bm W, (\bm x_t, y_t)).
\]
In addition, observe that $\max(0, l_\eta(\bm W, (\bm x_t, y_t)))$ equals $\ell_\eta(\bm W, (\bm x_t, y_t))$.

We first give a claim that provides the subgradient norm bound and the exp-concave
parameter of the loss $\tilde{\ell}_t(W)$. We defer its proof to the end of this section.
\begin{claim}
  %If at round $t$, $\tilde{y}_t = y_t$ and $\hat{y}_t \neq y_t$, then
  Suppose positive constants $D$, $X$ and $\eta$ satisfies that
  $\eta \leq \frac{1}{\max(2XD, 1)}$. In addition, suppose $\| \bm x_t\| \leq X$. Then,
  for all $\bm W$ such that $\| \bm W\|_F \leq D$, the loss function $\tilde{\ell}_t(\bm W)$ is
  $\frac{\gamma}{k} \cdot \frac{\eta}{32}$-exp-concave,
  and its subgradient has norm at most $\frac{k}{\gamma} \cdot 16 X$.
  \label{claim:parameters}
\end{claim}
Combining the above claim with the properties of the online optimization oracle $\calO$ (Equation~\eqref{eqn:expconcave}), we get that
\[
\sum_{t=1}^T \tilde{\ell}_t(\bm W_t, (\bm x_t, y_t))  - \sum_{t=1}^T \tilde{\ell}_t(\bm U, (\bm x_t, y_t)) \leq
O \left(  (16 X D  + \frac {32} \eta) \cdot \frac{k}{\gamma} \cdot kd \ln T \right) = O \left(  \frac{k^2 d}{\gamma \eta} \ln T \right). \]
Taking expectation on both sides yields that
\[ \E[\sum_{t=1}^T \one[\hat{y}_t \neq y_t] \cdot l_\eta(\bm W_t, (\bm x_t, y_t))] - \sum_{t=1}^T \one[\hat{y}_t \neq y_t] \cdot l_\eta(\bm U, (\bm x_t, y_t))
\leq O \left(  \frac{k^2 d}{\gamma \eta} \ln T \right), \]
Observe that when $\hat{y}_t \neq y_t$, $l_\eta(\bm W_t, (\bm x_t, y_t)) = \ell_\eta(\bm W_t, (\bm x_t, y_t))$; in addition, $\one[\hat{y}_t \neq y_t] \cdot l_\eta(\bm U, (\bm x_t, y_t)) \leq \one[\hat{y}_t \neq y_t] \cdot \ell_\eta(\bm U, (\bm x_t, y_t)) \leq \ell_\eta(\bm U, (\bm x_t, y_t))$.
Plugging the above facts into the inequality,
we establish the first item.

For the second item, we first use the fact that for each $t$, $\ell_\eta(\bm W_t, (\bm x_t, y_t)) \geq \one[\hat{y}_t \neq y_t]$, getting
\[ \E[\sum_{t=1}^T \one[\hat{y}_t \neq y_t]] - \sum_{t=1}^T \ell_\eta(\bm U, (\bm x_t, y_t))
\leq  O \left(  \frac{k^2 d}{\gamma \eta} \ln T \right). \]
Next, using the triangle inequality that $\one[\tilde{y}_t \neq y_t] \leq \one[\tilde{y}_t \neq \hat{y}_t] + \one[\hat{y}_t \neq y_t]$, we have that
\[ \E[M_T] - \sum_{t=1}^T \ell_\eta(\bm U, (\bm x_t, y_t))
\leq  O \left(  \frac{k^2 d}{\gamma \eta} \ln T \right) + \E[\sum_{t=1}^T \one[\tilde{y}_t \neq \hat{y}_t]] = O \left(  \frac{k^2 d}{\gamma \eta} \ln T  + \gamma T \right). \]
Plugging $\gamma = \sqrt{\frac {k^2 d \ln T} {\eta T} }$ into the above bound immediately gives the second item.
\end{proof}

We now come back to prove Claim~\ref{claim:parameters}.
\begin{proof}[Proof of Claim~\ref{claim:parameters}]
If $\tilde{y}_t \neq y_t$ or $\hat{y}_t \neq y_t$, then $\tilde{\ell}_t(\bm W)$ satisfies the
exp-concavity and bounded subgradient properties trivially.

Otherwise, $\tilde{y}_t = y_t$ and $\hat{y}_t \neq y_t$. In this case, by the definition of $\bm p_t$, we always have $p_{t,y_t} = \frac{\gamma}{k}$. Thus, $\tilde{\ell}_t(W) = \frac{k}{\gamma} \cdot l_\eta(\bm W, (\bm x_t, y_t))$.
It therefore suffices to show that the function $\bm W \mapsto l_\eta(\bm W, (\bm x_t, y_t))$ is $\frac{\eta}{32}$-exp-concave, and its subgradient has norm at most $16 X$.

  We can rewrite $l_\eta(\bm W, (\bm x_t, y_t))$ as the composition of functions $f(m) = (1 - \frac{2}{2-\eta} m + \frac{\eta}{2-\eta} m^2 )$ and $m_t(\bm W) = (\bm W \bm x_t)_{y_t} - \max_{y \neq y_t} (\bm W \bm x_t)_t$. Observe that
  $f_t$ is monotonically decreasing in $(-\infty, \frac 1 \eta)$, which is a superset of $[-2XD, 2XD]$ by the assumption on $\eta$.
  Therefore, $\ell_\eta(\bm W, (\bm x_t, y_t))$ can be written as
  $\max_{y \neq y_t} (1 - \frac{2}{2-\eta} ((\bm W \bm x_t)_{y_t} - (\bm W \bm x_t)_y) + \frac{\eta}{2-\eta} ((\bm W \bm x)_{y_t} - (\bm W \bm x)_y)^2)$. As subgradient norm property and exp-concavity are preserved under pointwise maximum over functions, it suffices to show that for every $y \neq y_t$, $(1 - \frac{2}{2-\eta} ((\bm W \bm x_t)_{y_t} - (\bm W \bm x_t)_y) + \frac{\eta}{2-\eta} ((\bm W \bm x_t)_{y_t} - (\bm W \bm x)_y)^2)$ is
  $\frac{\eta}{32}$-exp-concave, and its subgradient has norm at most $16 X$.

  Similar to the previous reasoning, function $\bm W \mapsto l_{t,y}(\bm W, (\bm x_t, y_t)) := (1 - \frac{2}{2-\eta} ((\bm W \bm x_t)_{y_t} - (\bm W \bm x_t)_y) + \frac{\eta}{2-\eta} ((\bm W \bm x_t)_{y_t} - (\bm W \bm x_t)_y)^2)$ is the composition of function $f_t$ and a linear function $m_{t,y}(\bm W) = (\bm W \bm x)_{y_t} - (\bm W \bm x)_y$. We first show that $f_t$ is
  $\frac{\eta}{32}$-exp-concave. This follows from the fact that for $m \in [-2XD, 2XD]$,
  \[ |f_t'(m)| = |-\frac 2 {2-\eta} + \frac{2m\eta}{2-\eta}| \leq 4, \]
  \[ f_t''(m) = \frac{\eta}{2-\eta} \geq \frac \eta 2. \]
  Hence, $\frac{f_t''(m)}{(f_t'(m))^2} \geq \frac \eta {32}$.
  It follows that
  \[ \frac{d^2}{dm^2} \exp(-\frac{\eta}{32} f_t(m)) = ((\frac{\eta}{32} f'_t(m) )^2 - \frac{\eta}{32} f''_t(m)) \exp(-\frac{\eta}{32} f_t(m)) \leq 0,  \]
  proving that $f_t$ is $\frac{\eta}{32}$-exp-concave.
  As exp-concavity is preserved under linear transformation on inputs, and $m_{t,y}(\cdot)$ is linear in $\bm W$, $l_{t,y}(\bm W, (\bm x_t, y_t))$ is $\frac{\eta}{32}$-exp-concave.

  We next bound the norm of $\frac{ \partial l_{t,y}(\bm W, (\bm x_t, y_t))}{\partial \bm W}$. By the chain rule, it is equal to $|f_t'(m)| \cdot \|\frac{\partial m}{\partial \bm W}\|$, and is consequently at most $16 X$.
  The claim follows.
\end{proof}